\documentclass[11pt]{article}

\usepackage[final]{acl}

\usepackage{times}
\usepackage{latexsym}

\usepackage[T1]{fontenc}

\usepackage[utf8]{inputenc}

\usepackage{microtype}

\usepackage{inconsolata}

\usepackage{graphicx}

\usepackage{dsfont,booktabs,bm,multirow,array,xcolor}
\usepackage{amsmath,amssymb,bm}
\usepackage{tikz}
\usetikzlibrary{positioning,arrows.meta,calc,fit,backgrounds,shapes.geometric,shapes.misc,decorations.pathreplacing}

\title{LEAF: Growing Trees Without Branching for Speech-Aware Large Language Model Post-Training}

\author{%
  Argyrios Gerogiannis\thanks{Corresponding author.}\\{\normalsize\texttt{ag91@illinois.edu}} \And
  Yekaterina Yegorova\\{\normalsize\texttt{yay2@illinois.edu}} \\
  \AND
  Mark Hasegawa-Johnson\\{\normalsize\texttt{jhasegaw@illinois.edu}} \And
  Venugopal V. Veeravalli\\{\normalsize\texttt{vvv@illinois.edu}} \\
  \AND
  \textnormal{{The Grainger College of Engineering}} \\
  \textnormal{{University of Illinois Urbana-Champaign}}
}

\newcommand{\best}[1]{\textbf{#1}}

\usepackage[dvipsnames]{xcolor}

\usepackage{soul} 
\usepackage{multirow}

\usepackage{amsthm}

\newtheorem{theorem}{Theorem}
\newtheorem{remark}{Remark}
\newtheorem{assumption}{Assumption}

\begin{document}
\maketitle
\begin{abstract}
State-of-the-art GRPO-style methods for speech-aware large language model post-training suffer from coarse credit assignment, broadcasting the same terminal-reward advantage to every token in a response. This ignores useful structure within rollout batches, where speech-conditioned completions often share prefixes before diverging at important decisions. We propose Low-rank Exploration with Adaptive Forking (LEAF), a retrospective tree-based RL method that recovers this structure without online branching or additional decoding. LEAF samples complete responses, selects high-surprisal boundaries, groups responses by shared prefixes, and assigns span-level advantages using descendant rewards. We theoretically justify LEAF's span-level credit assignment and boundary-selection design. Empirically, LEAF improves over GRPO across speech question answering and speech translation benchmarks under the same rollout and low-rank adaptation budget. Notably, smaller LEAF-trained models consistently outperform current state-of-art, full-parameter baselines.
\end{abstract}

\section{Introduction}
\label{sec:intro}

Speech-aware large language models (SALLMs) \cite{arora2025sallmsurvey} enable end-to-end spoken interaction across multiple tasks such as spoken question answering (SQA), automatic speech translation (AST), and speech recognition. As these models improve, post-training becomes essential for aligning generation with task-level objectives not captured by next-token prediction alone. Reinforcement learning (RL) is a natural fit, and recent GRPO-style methods \cite{shao2024deepseekmathpushinglimitsmathematical,yu2025dapo} have shown that reward-driven adaptation can improve open-form speech understanding while avoiding additional training costs by normalizing terminal rewards across multiple completions.

Despite its simplicity, GRPO is limited by coarse credit assignment: it computes
a terminal reward per response and broadcasts an identical scalar advantage to
every token. While efficient, this discards \emph{where} a completion becomes
promising or harmful. The issue is especially pronounced in open-form speech
post-training, where candidate answers often share early segments and diverge only around uncertain semantic decisions, so the bottleneck is not reward design
alone but \emph{rollout organization} and \emph{credit assignment}.

Tree-structured RL offers a finer alternative: by branching from intermediate
states and using descendant outcomes for process-level supervision, it localizes
credit to the decisions that matter. Directly importing online tree expansion
into SALLM post-training is costly however, since speech inputs are long and
multimodal, decoding is batched, and rewards are task-specific. This motivates
our central question: \emph{can we obtain the benefits of structured exploration
and process-aware credit assignment while retaining the simplicity of batched
GRPO-style training for SALLMs?}

We answer with \textbf{L}ow-rank \textbf{E}xploration with \textbf{A}daptive \textbf{F}orking (\textbf{LEAF}), a retroactive RL method for SALLM post-training. LEAF samples the same $K$ complete responses as GRPO, then reorganizes them into a sparse retrospective prefix tree. It selects  a small number of high-surprisal token boundaries, groups completions by exact shared prefixes at those boundaries, estimates prefix values from descendant terminal rewards, and assigns span-level advantages between retained nodes. LEAF therefore keeps GRPO's batched generation pattern while replacing sequence-level credit with structured span-level supervision derived from the rollout batch itself. Finally, LEAF is designed for parameter-efficient speech post-training. We instantiate it with LoRA \cite{hu2022lora}, motivated by recent evidence that low-rank adaptation can approach full-parameter RL post-training performance \citep{schulman2025lora}. This lets us test whether better credit assignment can translate into stronger SALLM adaptation under a lightweight trainable-parameter budget. To our knowledge, LEAF
is the first tree-based RL method for speech post-training, and the first post-training method within
retrospective credit assignment to use an uncertainty-guided mechanism for
constructing prefix-level supervision.

\noindent\textbf{Contributions.} We introduce LEAF, a retroactive tree-based RL method for SALLM post-training that extracts process-aware span advantages from the same i.i.d. rollout groups used by GRPO. We provide theoretical support for LEAF's design, showing that span-level prefix advantages are valid credit signals, that the fork budget controls a support-resolution tradeoff, and that pure prefix matching is biased toward shallow high-collision prefixes. Empirically, we show that speech-understanding rollouts are highly forkable across tasks and backbones, and that LEAF consistently improves over GRPO under the same rollout and LoRA adaptation budget, with gains in automatic metrics, judge scores and tail-risk behavior. The implementation of LEAF is available at: \texttt{https://github.com/Grimoid/LEAF}.

\section{Related Work}
\label{sec:related}

\textbf{GRPO for SALLMs.}
GRPO \citep{shao2024deepseekmathpushinglimitsmathematical} is a critic-free RL objective that normalizes terminal rewards within a group of sampled completions. Recent work applies GRPO to SALLMs for open-form speech understanding, including SQA and AST \citep{elmakies2025speechgrpo}, as well as to automatic speech recognition \citep{shivakumar2025grpoasr} and text-to-speech \citep{liu2025grpotts}. These works show that reward-driven post-training is effective for speech, but retain GRPO's flat credit assignment: one terminal-reward advantage is broadcast to all tokens in a response. LEAF keeps the same group-based rollout interface, but replaces sequence-level credit with span-level credit derived from shared prefixes.

\noindent\textbf{Fine-grained and tree-structured credit.}
Several recent methods seek denser credit than trajectory-level rewards in LLM RL. VinePPO estimates intermediate values by sampling continuations from partial states~\citep{kazemnejad2025vineppo}, SPO assigns advantages at the segment level~\citep{guo2026spo}, and high-entropy analyses suggest that a few uncertain tokens often drive downstream decisions~\citep{wang2026highentropy}. Tree-based
methods obtain process-level supervision by organizing rollouts into branching structures~\citep{li2025treepo,yang2025treerpo,ji2026treegrpollms,hou2025treerl}. LEAF is closest to TreeRL and TEMPO/P2T~\citep{hou2025treerl,tran2025tempo}, which estimate prefix-level credit from a fixed group of sampled completions, but it differs in setting, construction, and credit assignment: it performs no additional decoding, targets speech rather than text LLMs, and forks only at a sparse set of high-surprisal locations rather than building the full prefix tree by exact token matching. We formalize in Section~\ref{sec:theory} why such collision-only matching is problematic.

\section{Method}
\label{sec:method}

LEAF departs from the dominant execution model in tree-structured RL, which expands a search tree iteratively at inference time. Instead, it exploits structure already present in the batch of i.i.d.\ traces collected during training: candidate responses that share prefixes are organized retrospectively into a prefix tree, with no additional rollouts. This tree replaces GRPO's advantage computation. Rather than assigning a single scalar advantage to every token in a response, LEAF assigns piecewise-constant advantages over token spans, localizing credit to where completions diverge (Figure~\ref{fig:leaf_overview}).

\subsection{Preliminaries: GRPO}
\label{subsec:grpo-prelim}
Let $x$ be a speech-conditioned input, $\pi_{\mathrm{old}}$ the rollout policy, $\pi_\theta$ the optimized policy, and $\pi_{\mathrm{ref}}$ a fixed reference policy. GRPO samples $K$ responses $y^{(1)},\ldots,y^{(K)}\sim\pi_{\mathrm{old}}(\cdot\mid x)$ and assigns terminal rewards $r_i=r(x,y^{(i)})$. It forms the group-relative advantage $A_i^{\mathrm{GRPO}}=(r_i-\bar r)/\sigma_r$, where $\bar r$ and $\sigma_r$ are the sample mean and standard deviation, and broadcasts this scalar to every generated token in response $i$. For token $a_t^{(i)}$ generated from state $s_t^{(i)}=(x,y_{<t}^{(i)})$, define $\rho_{i,t}(\theta)=\pi_\theta(a_t^{(i)}|s_t^{(i)})/\pi_{\mathrm{old}}(a_t^{(i)}| s_t^{(i)})$. Given token-level advantages $A_{i,t}$, let $\ell_{i,t}(\theta)=\min\{\rho_{i,t}A_{i,t},\operatorname{clip}_{\epsilon}(\rho_{i,t})A_{i,t}\}$. The clipped KL-regularized loss is
$
\mathcal{L}(\theta)
=
-\mathbb{E}_{i,t}\!\left[\ell_{i,t}(\theta)\right]
+
\beta\,\mathbb{E}_{i,t}\!\left[\widehat{\mathrm{KL}}_{i,t}\right],
\label{eq:method-loss}
$
where $\widehat{\mathrm{KL}}_{i,t}$ estimates the token-level divergence from $\pi_{\mathrm{ref}}$ and
$\operatorname{clip}_{\epsilon}(\rho)$ truncates $\rho$ to the range $[1-\epsilon,1+\epsilon]$. GRPO is recovered by setting $A_{i,t}=A_i^{\mathrm{GRPO}}$ for all $t$.

\subsection{Retrospective Tree Credit Assignment}
\label{subsec:leaf-overview}

\begin{figure*}[t]
\centering
\resizebox{\textwidth}{!}{%
\begin{tikzpicture}[
    font=\small,
    >=Latex,
    tok/.style={draw, rounded corners=1.2pt, minimum width=0.35cm, minimum height=0.28cm, inner sep=0pt},
    blueTok/.style={tok, fill=blue!30},
    orangeTok/.style={tok, fill=orange!35},
    greenTok/.style={tok, fill=green!30},
    grayTok/.style={tok, fill=gray!25},
    redTok/.style={tok, draw=red!80, fill=red!12, line width=0.9pt},
    deltaTok/.style={tok, draw=red!70, dashed, fill=red!8, line width=0.9pt},
    budgetTok/.style={tok, draw=red!70, densely dotted, fill=red!8, line width=0.9pt},
    nodeBox/.style={draw, rounded corners=3pt, align=center, inner sep=3pt, font=\scriptsize},
    panelTitle/.style={font=\bfseries},
    note/.style={font=\scriptsize, align=center},
    arrow/.style={-{Latex[length=2mm]}, thick},
    thinArrow/.style={-{Latex[length=1.6mm]}, semithick},
    branchArrow/.style={-{Latex[length=1.7mm]}, semithick},
    spanBrace/.style={decorate, decoration={brace, amplitude=3pt, mirror}, thick},
    advBraceL/.style={decorate, decoration={brace, amplitude=3pt, mirror}, semithick},
    advBraceR/.style={decorate, decoration={brace, amplitude=3pt}, semithick},
    surp/.style={circle, draw=red!80, line width=0.9pt, inner sep=1.0pt, font=\tiny\bfseries, text=red!80},
    deltaSurp/.style={circle, draw=red!70, dashed, line width=0.9pt, inner sep=1.0pt, font=\tiny\bfseries, text=red!70},
    budgetSurp/.style={circle, draw=red!70, densely dotted, line width=0.9pt, inner sep=1.0pt, font=\tiny\bfseries, text=red!70}
]

\node[panelTitle] at (1.85,2.40) {Scan surprisals and select boundaries};
\node[panelTitle] at (16.00,2.7) {Prefix tree, retained nodes, and span credit};

\draw[gray!28, line width=0.35pt] (7.4,2.55) -- (7.4,-7.20);

\begin{scope}[xshift=-15mm]

\foreach \lbl/\y in {1/1.20,5/0.35,2/-0.50,3/-1.35,4/-2.20,6/-3.05} {
    \node[anchor=east] at (-1.15,\y) {$y^{(\lbl)}$};
}

\foreach \x/\tt in {-0.60/0,-0.25/1,0.10/2,0.45/3,0.80/4,1.15/5,1.50/6,1.85/7,2.20/8} {
    \node[font=\tiny, text=gray!78] at (\x,1.92) {\tt};
}
\node[font=\tiny, text=gray!78, anchor=east] at (-0.90,1.92) {$t=$};

\node[note, text=red!75!black] at (1.00,1.62)
{greedy scan in descending surprisal};

\foreach \x in {-0.60,-0.25} \node[grayTok] at (\x,1.20) {};
\foreach \x in {0.10,0.45,0.80,1.15} \node[blueTok] at (\x,1.20) {};
\foreach \x in {1.50,1.85,2.20} \node[blueTok, fill=blue!12] at (\x,1.20) {};

\foreach \x in {-0.60,-0.25} \node[grayTok] at (\x,0.35) {};
\foreach \x in {0.10,0.45,0.80,1.15} \node[orangeTok] at (\x,0.35) {};
\foreach \x in {1.50,1.85,2.20} \node[orangeTok, fill=orange!15] at (\x,0.35) {};

\foreach \x in {-0.60,-0.25} \node[grayTok] at (\x,-0.50) {};
\foreach \x in {0.10,0.45,0.80,1.15} \node[blueTok] at (\x,-0.50) {};
\foreach \x in {1.50,1.85,2.20} \node[blueTok, fill=blue!12] at (\x,-0.50) {};

\foreach \x in {-0.60,-0.25} \node[grayTok] at (\x,-1.35) {};
\foreach \x in {0.10,0.45,0.80,1.15} \node[blueTok] at (\x,-1.35) {};
\foreach \x in {1.50,1.85,2.20} \node[blueTok, fill=blue!12] at (\x,-1.35) {};

\foreach \x in {-0.60,-0.25} \node[grayTok] at (\x,-2.20) {};
\foreach \x in {0.10,0.45,0.80,1.15} \node[orangeTok] at (\x,-2.20) {};
\foreach \x in {1.50,1.85,2.20} \node[orangeTok, fill=orange!15] at (\x,-2.20) {};

\foreach \x in {-0.60,-0.25} \node[grayTok] at (\x,-3.05) {};
\foreach \x in {0.10,0.45,0.80,1.15} \node[greenTok] at (\x,-3.05) {};
\foreach \x in {1.50,1.85,2.20} \node[greenTok, fill=green!12] at (\x,-3.05) {};

\node[redTok]    at (1.50,1.20) {};
\node[deltaTok]  at (1.85,0.35) {};
\node[redTok]    at (0.10,-0.50) {};
\node[budgetTok] at (2.20,-3.05) {};

\node[surp]       at (1.50,1.20) {!};
\node[deltaSurp]  at (1.85,0.35) {$\Delta$};
\node[surp]       at (0.10,-0.50) {!};
\node[budgetSurp] at (2.20,-3.05) {$B$};

\node[note] at (0.80,-3.65)
{sample $K$ completed responses\\ and scan token surprisals};

\node[nodeBox, fill=gray!10, minimum width=4.15cm] (scan0) at (5.70,1.10)
{\textbf{sorted candidates by surprisal}\\
$(2,2),\ (1,6),\ (5,7),\ (6,8)$};

\node[nodeBox, fill=green!8, minimum width=4.15cm] (scan1) at (5.70,0.00)
{\textbf{scan $(2,2)$}: propose $p=2$\\
\textbf{accept} $\Rightarrow \mathcal P=\{2\}$};

\node[nodeBox, fill=green!8, minimum width=4.15cm] (scan2) at (5.70,-1.10)
{\textbf{scan $(1,6)$}: propose $p=6$\\
\textbf{accept} $\Rightarrow \mathcal P=\{2,6\}$};

\node[nodeBox, fill=orange!12, minimum width=4.15cm] (scan3) at (5.70,-2.20)
{\textbf{scan $(5,7)$}: propose $p=7$\\
\textbf{reject by $\Delta$}: $|7-6|<\Delta$};

\node[nodeBox, fill=gray!12, minimum width=4.15cm] (scan4) at (5.70,-3.30)
{\textbf{scan $(6,8)$}: propose $p=8$\\
\textbf{skip}: $|\mathcal P|=B=2$};

\node[nodeBox, fill=gray!5, minimum width=4.15cm] (finalP) at (5.70,-4.40)
{\textbf{selected global boundaries}\\
$\mathcal P=\{2,6\}$};

\draw[arrow] (scan0) -- (scan1);
\draw[arrow] (scan1) -- (scan2);
\draw[arrow] (scan2) -- (scan3);
\draw[arrow] (scan3) -- (scan4);
\draw[arrow] (scan4) -- (finalP);

\draw[arrow, gray!62] (0.33,-0.50) -- (3.55,0.00);
\draw[arrow, gray!62] (1.73,1.20)  -- (3.55,-1.10);
\draw[arrow, gray!62] (2.08,0.35)  -- (3.55,-2.20);
\draw[arrow, gray!62] (2.43,-3.05) -- (3.55,-3.30);

\node[note] at (5.70,-5.25)
{$\Delta$ avoids nearby boundaries\\ $B$ caps how many are retained};

\end{scope}

\begin{scope}

\begin{scope}[xshift=1.0cm]

\node[nodeBox, fill=gray!15, minimum width=1.60cm] (vroot) at (12.65,1.95)
{$v_0$\\root};

\node[draw, line width=0.8pt, rounded corners=3pt, fill=gray!8,
      minimum width=0.78cm, minimum height=1.10cm] (v2strip) at (12.65,0.82) {};
\foreach \y in {1.00,0.64} \node[grayTok] at (12.65,\y) {};
\node[font=\small, anchor=west] at (13.08,0.82) {$v_2$};

\draw[arrow] (vroot.south) -- (v2strip.north);

\node[draw, line width=0.8pt, rounded corners=3pt, fill=blue!10,
      minimum width=0.82cm, minimum height=1.70cm] (vbstrip) at (9.70,-1.05) {};
\foreach \y in {-0.47,-0.82,-1.17,-1.52} \node[blueTok] at (9.70,\y) {};
\node[font=\small, text=blue!55!black, anchor=west] at (10.15,-1.05) {$v_b$};

\node[draw, line width=0.8pt, rounded corners=3pt, fill=orange!12,
      minimum width=0.82cm, minimum height=1.70cm] (vostrip) at (12.65,-1.05) {};
\foreach \y in {-0.47,-0.82,-1.17,-1.52} \node[orangeTok] at (12.65,\y) {};
\node[font=\small, text=orange!60!black, anchor=west] at (13.10,-1.05) {$v_o$};

\node[draw, line width=0.8pt, rounded corners=3pt, fill=green!6,
      minimum width=0.62cm, minimum height=2.60cm] (y6box) at (16.05,-1.50) {};

\foreach \y in {-0.47,-0.82,-1.17,-1.52}
    \node[greenTok] at (16.05,\y) {};

\foreach \y in {-1.85,-2.18,-2.51}
    \node[greenTok, fill=green!12] at (16.05,\y) {};

\node[draw, dashed, line width=0.8pt, rounded corners=3pt, green!45!black,
      minimum width=0.48cm, minimum height=1.50cm] (vgbox) at (16.05,-0.995) {};

\node[font=\small, text=green!45!black, anchor=west] at (16.48,-0.995) {$u_g$};

\draw[branchArrow] ([xshift=-0.20cm]v2strip.south) -- (vbstrip.north);
\draw[branchArrow] (v2strip.south) -- (vostrip.north);
\draw[branchArrow, green!45!black] ([xshift=0.20cm]v2strip.south) -- (y6box.north);

\node[draw, line width=0.8pt, rounded corners=3pt, fill=blue!5,
      minimum width=0.58cm, minimum height=1.35cm] (y1box) at (8.10,-3.45) {};
\foreach \y in {-3.10,-3.45,-3.80}
    \node[blueTok, fill=blue!12] at (8.10,\y) {};
\node[note] at (8.10,-4.42) {$y^{(1)}$};

\node[draw, line width=0.8pt, rounded corners=3pt, fill=blue!5,
      minimum width=0.58cm, minimum height=1.35cm] (y2box) at (9.70,-3.45) {};
\foreach \y in {-3.10,-3.45,-3.80}
    \node[blueTok, fill=blue!12] at (9.70,\y) {};
\node[note] at (9.70,-4.42) {$y^{(2)}$};

\node[draw, line width=0.8pt, rounded corners=3pt, fill=blue!5,
      minimum width=0.58cm, minimum height=1.35cm] (y3box) at (11.35,-3.45) {};
\foreach \y in {-3.10,-3.45,-3.80}
    \node[blueTok, fill=blue!12] at (11.35,\y) {};
\node[note] at (11.35,-4.42) {$y^{(3)}$};

\draw[thinArrow] ([xshift=-0.26cm]vbstrip.south) -- (y1box.north);
\draw[thinArrow] (vbstrip.south) -- (y2box.north);
\draw[thinArrow] ([xshift=0.26cm]vbstrip.south) -- (y3box.north);

\node[draw, line width=0.8pt, rounded corners=3pt, fill=orange!8,
      minimum width=0.58cm, minimum height=1.35cm] (y4box) at (13.10,-3.45) {};
\foreach \y in {-3.10,-3.45,-3.80}
    \node[orangeTok, fill=orange!15] at (13.10,\y) {};
\node[note] at (13.10,-4.42) {$y^{(4)}$};

\node[draw, line width=0.8pt, rounded corners=3pt, fill=orange!8,
      minimum width=0.58cm, minimum height=1.35cm] (y5box) at (14.70,-3.45) {};
\foreach \y in {-3.10,-3.45,-3.80}
    \node[orangeTok, fill=orange!15] at (14.70,\y) {};
\node[note] at (14.70,-4.42) {$y^{(5)}$};

\draw[thinArrow] ([xshift=-0.20cm]vostrip.south) -- (y4box.north);
\draw[thinArrow] ([xshift=0.20cm]vostrip.south) -- (y5box.north);

\node[note] at (16.05,-3.10) {$y^{(6)}$};


\draw[advBraceL]
([xshift=-0.18cm]v2strip.north west) -- ([xshift=-0.18cm]v2strip.south west)
node[midway, xshift=-0.46cm, font=\scriptsize] {$A(v_2)$};

\draw[advBraceL, blue!55!black]
([xshift=-0.18cm]vbstrip.north west) -- ([xshift=-0.18cm]vbstrip.south west)
node[midway, xshift=-0.48cm, font=\scriptsize, text=blue!55!black] {$A(v_b)$};

\draw[advBraceL, orange!60!black]
([xshift=-0.18cm]vostrip.north west) -- ([xshift=-0.18cm]vostrip.south west)
node[midway, xshift=-0.48cm, font=\scriptsize, text=orange!60!black] {$A(v_o)$};

\draw[advBraceL, blue!55!black]
([xshift=-0.16cm]y1box.north west) -- ([xshift=-0.16cm]y1box.south west)
node[midway, xshift=-0.46cm, font=\scriptsize, text=blue!55!black]
{$A_{\mathrm{tail}}^{(1)}$};

\draw[advBraceL, blue!55!black]
([xshift=-0.16cm]y2box.north west) -- ([xshift=-0.16cm]y2box.south west)
node[midway, xshift=-0.46cm, font=\scriptsize, text=blue!55!black]
{$A_{\mathrm{tail}}^{(2)}$};

\draw[advBraceL, blue!55!black]
([xshift=-0.16cm]y3box.north west) -- ([xshift=-0.16cm]y3box.south west)
node[midway, xshift=-0.46cm, font=\scriptsize, text=blue!55!black]
{$A_{\mathrm{tail}}^{(3)}$};

\draw[advBraceL, orange!60!black]
([xshift=-0.16cm]y4box.north west) -- ([xshift=-0.16cm]y4box.south west)
node[midway, xshift=-0.46cm, font=\scriptsize, text=orange!60!black]
{$A_{\mathrm{tail}}^{(4)}$};

\draw[advBraceL, orange!60!black]
([xshift=-0.16cm]y5box.north west) -- ([xshift=-0.16cm]y5box.south west)
node[midway, xshift=-0.46cm, font=\scriptsize, text=orange!60!black]
{$A_{\mathrm{tail}}^{(5)}$};

\draw[advBraceL, green!45!black]
([xshift=-0.18cm]y6box.north west) -- ([xshift=-0.18cm]y6box.south west)
node[midway, xshift=-0.50cm, font=\scriptsize, text=green!45!black]
{$A_{\mathrm{tail}}^{(6)}$};

\end{scope}

\node[nodeBox, fill=gray!15, anchor=west, align=left, minimum width=5.25cm, inner sep=4pt] (leg2) at (18.35,1.35)
{\textbf{$v_2$ $\to$ Retained Node}\\[-1pt]
$(2,u_2),\ n=6$\\[-1pt]
$u_2=y_{<2}^{(1)}=y_{<2}^{(2)}=y_{<2}^{(3)}=\cdots=y_{<2}^{(6)}$\\[-1pt]
$\widehat V(v_2)=(r_1+\cdots+r_6)/6$\\[-1pt]
$\scriptstyle A(v_2)=
\frac{2(\widehat V(v_2)-\widehat V(v_0))}{\sqrt{6}}$};

\node[nodeBox, fill=blue!22, anchor=west, align=left, minimum width=5.25cm, inner sep=4pt] (legb) at (18.35,-0.60)
{\textbf{$v_b$ $\to$ Retained Node}\\[-1pt]
$(6,u_b),\ n=3$\\[-1pt]
$y_{<6}^{(1)}=y_{<6}^{(2)}=y_{<6}^{(3)}$\\[-1pt]
$\widehat V(v_b)=(r_1+r_2+r_3)/3$\\[-1pt]
$\scriptstyle A(v_b)=
\frac{(\widehat V(v_b)-\widehat V(v_0))
+(\widehat V(v_b)-\widehat V(v_2))}{\sqrt{3}}$};

\node[nodeBox, fill=orange!30, anchor=west, align=left, minimum width=5.25cm, inner sep=4pt] (lego) at (18.35,-2.55)
{\textbf{$v_o$ $\to$ Retained Node}\\[-1pt]
$(6,u_o),\ n=2$\\[-1pt]
$y_{<6}^{(4)}=y_{<6}^{(5)}$\\[-1pt]
$\widehat V(v_o)=(r_4+r_5)/2$\\[-1pt]
$\scriptstyle A(v_o)=
\frac{(\widehat V(v_o)-\widehat V(v_0))
+(\widehat V(v_o)-\widehat V(v_2))}{\sqrt{2}}$};

\node[nodeBox, fill=green!20, dashed, anchor=west, align=left, minimum width=5.25cm, inner sep=4pt] (legg) at (18.35,-4.35)
{\textbf{$u_g$ $\to$ Non-Retained Node}\\[-1pt]
$n=1$\\[-1pt]
$u_g=y_{<6}^{(6)}$\\[-1pt]
singleton prefix; no retained node $v_g$\\[-1pt]
$\scriptstyle A_{\mathrm{tail}}^{(6)}
=(r_6-\widehat V(v_0))+(r_6-\widehat V(v_2))$};

\node[note, font=\scriptsize\bfseries] at (14.50,-4.95)
{Span-credit examples using retained nodes};

\begin{scope}[xshift=1.00cm]

\node[note, anchor=west, text=blue!55!black] at (9.35,-5.50)
{\textbf{Blue path:} $v_0\to v_2\to v_b$};

\node[anchor=east] at (9.25,-6.25) {$y^{(1)}$};

\foreach \x in {9.55,9.90} \node[grayTok] at (\x,-6.25) {};
\foreach \x in {10.25,10.60,10.95,11.30} \node[blueTok] at (\x,-6.25) {};
\foreach \x in {11.65,12.00,12.35} \node[blueTok, fill=blue!12] at (\x,-6.25) {};

\draw[thick] (10.08,-6.49) -- (10.08,-6.01);
\draw[thick] (11.48,-6.49) -- (11.48,-6.01);

\node[note] at (10.08,-5.83) {$p=2$};
\node[note] at (11.48,-5.83) {$p=6$};

\draw[spanBrace] (9.55,-6.63) -- (10.08,-6.63);
\node[note] at (9.82,-6.94) {$A(v_2)$};

\draw[spanBrace] (10.08,-6.63) -- (11.48,-6.63);
\node[note] at (10.78,-6.94) {$A(v_b)$};

\draw[spanBrace] (11.48,-6.63) -- (12.35,-6.63);
\node[note] at (11.92,-6.94) {$A_{\mathrm{tail}}^{(1)}$};

\node[note, anchor=west, text=green!45!black] at (14.55,-5.50)
{\textbf{Green path:} $v_0\to v_2$; $p=6$ is singleton};

\node[anchor=east] at (14.45,-6.25) {$y^{(6)}$};

\foreach \x in {14.75,15.10} \node[grayTok] at (\x,-6.25) {};
\foreach \x in {15.45,15.80,16.15,16.50} \node[greenTok] at (\x,-6.25) {};
\foreach \x in {16.85,17.20,17.55} \node[greenTok, fill=green!12] at (\x,-6.25) {};

\draw[thick] (15.28,-6.49) -- (15.28,-6.01);
\draw[dashed, thick, green!45!black] (16.68,-6.49) -- (16.68,-6.01);

\node[note] at (15.28,-5.83) {$p=2$};
\node[note, text=green!45!black] at (16.68,-5.83) {$p=6$};
\node[note, text=green!45!black] at (16.68,-6.03) {no $v_g$};

\draw[spanBrace] (14.75,-6.63) -- (15.28,-6.63);
\node[note] at (15.01,-6.94) {$A(v_2)$};

\draw[spanBrace] (15.28,-6.63) -- (17.55,-6.63);
\node[note] at (16.41,-6.94) {$A_{\mathrm{tail}}^{(6)}$};

\end{scope}

\end{scope}
\end{tikzpicture}
}
\caption{\textbf{Overview of LEAF.}
From the same $K$ completed responses, LEAF records token surprisals and greedily selects a small set of globally applied boundaries $\mathcal P$ using a separation rule. At each selected boundary $p$, responses are grouped by exact token-prefix equality, forming a tree over the selected boundaries. Non-singleton prefix groups become retained nodes whose values are averages of descendant terminal rewards, while singleton groups are not retained as advantage nodes. The retained nodes on each response path partition the response into spans, and LEAF assigns one advantage to each span.}
\label{fig:leaf_overview}
\end{figure*}
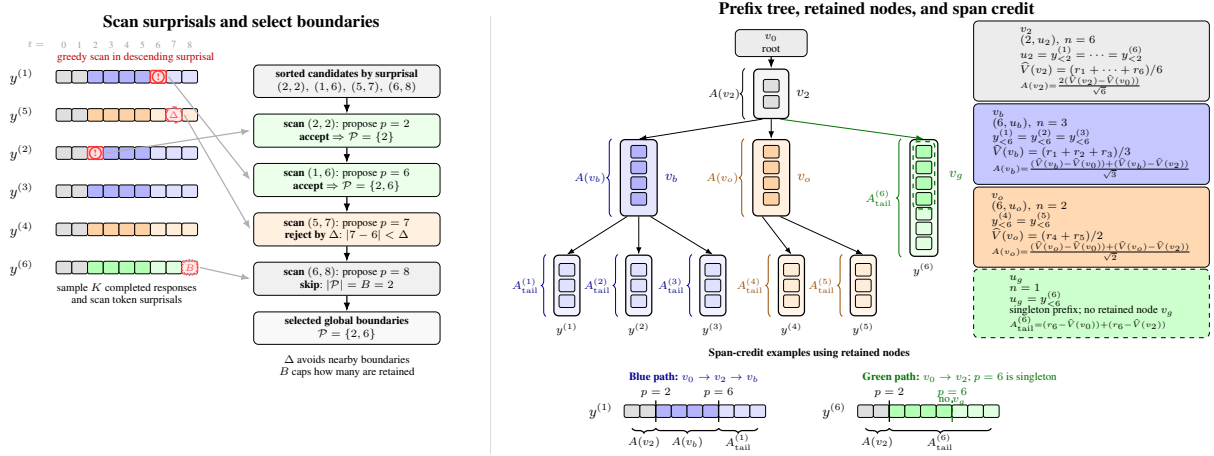

Conceptually, LEAF asks whether shared partial prefixes in the rollout batch predict better or worse final outcomes. It selects at most $B$ token boundaries, groups responses by exact prefix equality at each boundary, retains non-singleton groups as prefix nodes, and assigns advantages to the spans induced by the retained nodes. The only generation cost is the original $K$ responses; $B$ controls the resolution of credit assignment, not the rollout.

LEAF replaces only the advantage estimator $A_{i,t}$. The estimator is based on the process supervision of \citet{hou2025treerl}, which assigns intermediate credit using descendant leaf values, but differs in how the tree is obtained and extends to continuous rewards. The estimator is characterized in Section~\ref{sec:theory}. Rather than branching from intermediate prefixes and decoding new continuations, LEAF samples the same $K$ complete responses used by GRPO and constructs a prefix tree \emph{retrospectively} from shared prefixes within the sampled batch. Thus, the rollout budget remains exactly $K$ responses.

Let response $i$ be $y^{(i)}=(a_0^{(i)},a_1^{(i)},\ldots,a_{H_i-1}^{(i)})$, where $H_i$ is its length. For a token boundary $p$, define the length-$p$ prefix $y_{<p}^{(i)}=(a_0^{(i)},\ldots,a_{p-1}^{(i)})$. A selected boundary $p$ is applied to all $K$ responses: responses are grouped together if and only if their token prefixes $y_{<p}^{(i)}$ are exactly identical. Therefore, a selected boundary can produce zero, one, or multiple retained prefix nodes, depending on how many repeated prefix groups appear in the batch. LEAF replaces the advantage estimation as follows.

\textbf{Step 1: Responses and surprisals.}
We sample $K$ complete responses from $\pi_{\mathrm{old}}$ and compute their terminal rewards $r_1,\ldots,r_K$. During generation, for each $i\in[K]$, we record the token surprisal
$
    z_{i,t}
    =
    -\log \pi_{\mathrm{old}}(a_t^{(i)}\mid x,y_{<t}^{(i)}),
$
for $0\le t < H_i$. Surprisal is used only to decide where to inspect, without affecting the rollout distribution.

\textbf{Step 2: Fork boundaries.}
We use high-surprisal token positions to choose prefix boundaries. Let $\ell_{\min}=\min_{i\in[K]}H_i$ and $\Delta=\max\{2,\lfloor \ell_{\min}/(B+1)\rfloor\}$. The shortest length $\ell_{\min}$ restricts candidates to token boundaries that are present in every sampled response, so that the same boundary $p$ can be applied to all $K$ responses when forming exact prefix groups. The separation parameter $\Delta$ prevents the fork budget from being spent on several nearby high-surprisal tokens from the same local uncertainty region, encouraging the selected boundaries to cover distinct parts of the response. We form the candidate list
$
    \mathcal C
    =
    \{(z_{i,t},i,t): i\in[K],\ 1\le t < \ell_{\min}\},
$
and sort $\mathcal C$ in descending order of $z_{i,t}$. We initialize $\mathcal P=\emptyset$ and scan this sorted list greedily. When considering a candidate triple $(z_{i,t},i,t)$, we propose the boundary $p=t$, corresponding to the prefix immediately before the high-surprisal token $a_t^{(i)}$. The boundary $p$ is accepted if either $\mathcal P=\emptyset$ or $\min_{p'\in\mathcal P}|p-p'|\ge \Delta$. If accepted, we update $\mathcal P\leftarrow\mathcal P\cup\{p\}$. The scan stops once $|\mathcal P|=B$ or the candidate list is exhausted. Finally, we sort the accepted boundaries and write $\mathcal P=\{p_1,\ldots,p_b\}$ with $p_1<\cdots<p_b$ and $b\le B$. Accepting a boundary $p$ does not select a single response-specific prefix, as the same boundary $p$ is applied to all $K$ sampled responses, and grouping is performed by exact equality of the token sequences $y_{<p}^{(i)}$.

\textbf{Step 3: Retained prefix nodes.}
For each selected boundary $p\in\mathcal P$, we compare the length-$p$ token prefixes of all sampled responses. For a token prefix $u$, define the descendant index set $I(p,u)=\{i\in[K]:y_{<p}^{(i)}=u\}$, with $n(p,u)=|I(p,u)|$. The equality $y_{<p}^{(i)}=u$ is exact token-prefix matching: two responses are grouped together only if their first $p$ tokens are identical. If $n(p,u)\ge2$, the pair $v=(p,u)$ is retained as a prefix node. Singleton prefixes are discarded because they do not provide a reusable prefix-level value estimate. Let $\mathcal N_{\mathrm{fork}}$ denote the set of retained prefix nodes. For $v=(p,u)$, write $p(v)=p$, $u(v)=u$, $I(v)=I(p,u)$, and $n(v)=n(p,u)$. We also define a virtual root node $v_0$, corresponding to the empty prefix before any token is generated, with $p(v_0)=0$, $I(v_0)=[K]$, and $n(v_0)=K$. Its value is the batch-average reward $\widehat V(v_0)=\bar r$. Each retained node $v$ receives the empirical prefix value
$
    \widehat V(v)
    =
    \frac{1}{n(v)}\sum_{i\in I(v)} r_i.
$
Intuitively, $\widehat V(v)$ estimates how good completions tend to be after passing through the exact token prefix represented by node $v$, while $\widehat V(v_0)$ is the prompt-level average over all sampled completions.

\textbf{Step 4: Span-level advantages.}
For each response $i$, collect all retained nodes whose prefixes lie on the response:
$
    \mathcal N_i
    =
    \{v=(p,u)\in\mathcal N_{\mathrm{fork}}: y_{<p}^{(i)}=u\}.
$
Sort them by boundary position to obtain the path $v_1^{(i)},\ldots,v_{q_i}^{(i)}$, where $0<p(v_1^{(i)})<\cdots<p(v_{q_i}^{(i)})<H_i$, and $q_i$ is the number of retained prefix nodes on response $i$. Set $v_0^{(i)}=v_0$ and $p(v_0^{(i)})=0$. This path partitions response $i$ into spans between consecutive retained prefix nodes and a final tail span. For an internal span ending at node $v_j^{(i)}$, LEAF uses the global-plus-local advantage
$
    A_j^{(i)}
    =
    \left(\mathrm{GA}(v_j^{(i)})+\mathrm{LA}(v_j^{(i)})\right)/\sqrt{n(v_j^{(i)})},
$
where
$
    \mathrm{GA}(v_j^{(i)})
    =
    \widehat V(v_j^{(i)})-\widehat V(v_0)
$
and
$
    \mathrm{LA}(v_j^{(i)})
    =
    \widehat V(v_j^{(i)})-\widehat V(v_{j-1}^{(i)}).
$
$\mathrm{GA}$ measures whether the prefix is better than the prompt-level average, and $\mathrm{LA}$ measures whether the current span moved the response to a better prefix than the previous retained prefix on the same path. The factor $1/\sqrt{n(v_j^{(i)})}$ is a multiplicity correction: if many sampled responses share the same prefix, the same token span appears many times in the batch, and the correction prevents that duplicated span from dominating the update solely because it is repeated.

For each $j=1,\ldots,q_i$, every token $t\in[p(v_{j-1}^{(i)}),p(v_j^{(i)}))$ receives the same advantage $A_j^{(i)}$. Thus, the advantage is not constant over the whole response; it is piecewise constant over spans between retained prefix nodes. For the final tail span, we treat the terminal response as the leaf value. Let $v_{\mathrm{last}}^{(i)}=v_{q_i}^{(i)}$ if $q_i>0$, and $v_{\mathrm{last}}^{(i)}=v_0$ otherwise. We set
$
    A_{\mathrm{tail}}^{(i)}
    =
    (r_i-\widehat V(v_0))
    +
    (r_i-\widehat V(v_{\mathrm{last}}^{(i)})).
$
Every token $t\in[p(v_{\mathrm{last}}^{(i)}),H_i)$ receives $A_{\mathrm{tail}}^{(i)}$. If response $i$ has no retained prefix node, then the whole response is a single tail span and this gives $2(r_i-\widehat V(v_0))$, which has the same direction as the GRPO advantage before normalization. In practice, the raw span advantages are normalized across the sampled batch before entering the loss.

\section{Theoretical Insights}
\label{sec:theory}

We provide theoretical insights into three design choices of LEAF: the span-level advantage estimator, the fork budget $B$, and surprisal-based boundary selection. Specifically, we show that the empirical advantage estimator that LEAF uses preserves the correct update direction, we characterize the tradeoff regarding the value of $B$ and highlight the issues with pure prefix matching. We state concise informal versions of the results here and defer the formal statements and proofs to Appendix \ref{app:theory-proofs}.

\begin{theorem}[Validity of span-level advantages]
For any span ending at prefix $U_j$, the root-relative difference $V(U_j)-V(U_0)$ and the parent-relative difference $V(U_j)-V(U_{j-1})$ are valid scalar advantages. Under fixed-prefix conditional sampling, $\widehat V(v)$ is an unbiased estimate of $V(v)$ with variance scaling as $1/n(v)$. Thus, shared-prefix averaging denoises the value before span-level credit.
\label{thrm:adv_informal}
\end{theorem}

\begin{theorem}[Fork budget tradeoff]
The fork budget $B$ increases the number of selected boundaries at which prefix values may be estimated. With rollout budget $K$, at most $B\lfloor K/2\rfloor$ non-root prefix nodes can be retained. Increasing $B$ gives finer credit assignment but can create more low-support value estimates. With $K=8$ each boundary can retain at most four nodes; using $B=2$ keeps the worst-case number of non-root prefix values at $8$ while still allowing internal span-level credit.
\label{thrm:fork_budget_informal}
\end{theorem}

\begin{theorem}[Pure prefix matching is shallow-biased]
Prefix-collision mass is monotone non-increasing with boundary depth, so a selector that maximizes collisions alone selects from the earliest high-collision plateau. Thus, pure prefix matching can spend the fork budget on shallow boilerplate prefixes. LEAF separates the two roles: surprisal selects where to inspect, while exact prefix matching decides which prefixes at that boundary have enough support to retain.
\end{theorem}

\section{Experimental Setup}
\label{sec:experiments}

\textbf{Datasets.} We evaluate LEAF on four speech-language benchmarks covering short-form SQA, long-form SQA, and AST. For short-form SQA, we use Part I of LibriSQA \cite{zhao2023librisqa}, an open-ended QA benchmark over LibriSpeech audio \cite{panayotov2015librispeech}. For long-form SQA, we use two LongAudio-derived settings \cite{ghost2025longaudio}: DailyTalk \cite{lee2023dailytalk}, after removing multiple-choice examples, and a Europarl/VoxPopuli split \cite{koehn2005europarl,wang2021voxpopuli}, which we refer to as LongAudio. We restrict both long-form settings to audio segments of at most $40$ seconds. For AST, we use CoVoST2 English-to-German \cite{wang2020covost2}. We use dataset-provided instructions when available and only adapt input formatting to each SALLM; split details are in Appendix~\ref{app:datasets}.

\noindent \textbf{Baselines and backbones.}  We compare LEAF with the state-of-the-art (SOTA) GRPO speech-RL baseline of \citet{elmakies2025speechgrpo}. To our knowledge, this method represents the current SOTA on two of our datasets. For the LongAudio and DailyTalk splits, we are not aware of an established task-specific baseline, and hence, we use GRPO as the strongest applicable reference method. The reported results in \citet{elmakies2025speechgrpo} use full-parameter tuning; our controlled LEAF-vs.-GRPO experiments use LoRA for both methods to isolate the RL objective under the same trainable-parameter budget. When direct comparison to their reported numbers is possible, we include it; otherwise, we use our LoRA GRPO implementation. For rollouts, we match their reported budget with $K=8$ responses per prompt. LEAF uses the same rollout budget with fork budget $B=2$; all other optimization and decoding hyperparameters are matched between LEAF and GRPO and reported in Appendix~\ref{app:training_params}. We evaluate Granite Speech 3.3 2B/8B \cite{saon2025granitespeech3}, Granite-4.0-1B-Speech \cite{granite-4.0-1b-speech}, and Qwen2-Audio-7B \cite{chu2024qwen2audio}. Although Qwen is older than the SOTA Granite-4.0, we include it to evaluate whether LEAF generalizes across model families. 

\noindent \textbf{Evaluation metrics.} To ensure fair comparison with \citet{elmakies2025speechgrpo}, we report BLEU, ROUGE-1/2/L, METEOR, and BERTScore-F1 using temperature $0.9$ and top-$p=0.9$ decoding. To complement automatic metrics, we use M-Prometheus-14B \cite{pombal2025mprometheus} for rubric-grounded Likert-5 scoring and GEMBA-style DA-100 \cite{kocmi2023gembada100} with Qwen2.5-14B-Instruct \cite{qwen2025qwen25technicalreport} for continuous per-item scores. DA-100 scores are averaged over $N=5$ judgments at $T=0.7$ and used for tail-risk analysis. We also perform matched A/B comparisons between LEAF and GRPO. For QA tasks, we use JudgeLM-13B \cite{zhu2025judgelm}, while for CoVoST2 we use M-Prometheus-14B. Each pair is judged in both candidate orders, with ties allowed. The pairwise judges are decoded greedily.

\noindent \textbf{Statistical analysis.} We report $95\%$ confidence intervals for judge means and $\mathrm{CVaR}@\alpha$ using $1000$ non-parametric bootstrap resamples. Tail risk is measured by $\mathrm{CVaR}@\alpha$, the mean score over the worst $\alpha$ fraction of examples \cite{rockafellar2000cvar}. Because Likert-5 $\mathrm{CVaR}@10$ saturates at the minimum score whenever at least $10\%$ of examples receive score $1$, we use the continuous DA-100 scores for the main tail-risk analysis.

\section{When Does LEAF Work?}
\label{sec:why}

Before presenting the main results, we ask whether LEAF has the raw structure it needs to operate. It relies on an empirical property of i.i.d. rollout groups: sampled responses must share exact prefixes often enough, and at sufficiently deep token positions, for the induced virtual forks to provide useful credit-assignment signal. We call this property \emph{forkability}. To define it, consider $K$ rollouts for a fixed prompt. For a selected boundary $t$, group responses by exact equality of their length-$t$ prefixes. A \emph{fork-fire} occurs at $t$ if at least one prefix group has size at least two. A fork-fire is \emph{usable} if at least one such non-singleton group has non-trivial reward variation, $\max_{i,j\in G} |r_i-r_j| > 0$ for some retained prefix group $G$. Thus, usable fork-fires correspond to shared prefixes whose descendant responses lead to distinguishable outcomes, rather than merely duplicate completions. 

Open-form speech understanding is a natural regime for forkability: conditioned on the same audio, sampled responses can share substantial prefixes while diverging around localized acoustic or semantic decisions. If those alternative continuations produce different rewards, the resulting rollout structure provides precisely the intermediate signal that LEAF exploits. The converse also clarifies when LEAF should offer little advantage over GRPO. In very short-output tasks, such as direct multiple-choice prediction, responses may diverge before informative internal forks can form, leaving little opportunity for finer credit assignment. Forkability is not inherently specific to speech, however; we view it as a task- and policy-dependent condition that may also arise in other modalities. Establishing when it holds in settings such as long-form LLM reasoning is outside the scope of this work.

We measure forkability across four datasets and three backbones, a $3\times4$ design with $\approx 78{,}915$ fork attempts. The usable fork-fire rate is $\ge 72.9\%$ in every one of the twelve cases, while pooled rates differ by only $1.4$ pp across backbones. In contrast, rates vary substantially across tasks, suggesting that forkability in our setting is primarily driven by task structure rather than backbone choice.

\textbf{Direct measurement.} We run LEAF
from each of the twelve cells and run it on a held-out validation+test
split of $3000-3400$ examples per cell.
Table~\ref{tab:fork-headline} reports the pooled rates with $95\%$
per-position bootstrap CIs. The
pooled-by-backbone rates differ by $1.4$ pp.,
recovering backbone-invariance. The pooled-by-task rates span
$19.2$ pp and follow task difficulty (easier QA tasks fork more reliably than translation and long-form QA). Across all twelve cells,
$r_f$ never drops below $72.9\%$; the per-cell breakdown
(Tab.~\ref{tab:fork-percell-app}) is in the appendix.
Figure~\ref{fig:fork-fire-position} plots usable rate vs.\ token
position; for each task, the three backbone curves overlap within their
$95\%$ CIs at every position, giving an analogue of the
pooled claim.

\begin{table}[t]
\centering
\caption{Usable fork-fire rate $r_f$, pooled by backbone (top) and by task
(middle). Confidence intervals are per-position non-parametric bootstrap
($n_{\text{boot}}\!=\!5000$); positions are the unit of resampling
because rollouts within an example are correlated through the sampler.
}
\scriptsize
\setlength{\tabcolsep}{4pt}
\begin{tabular}{lr}\toprule
& Usable fork-fire rate $r_f$ (\%, [95\% CI]) \\\midrule
\multicolumn{2}{l}{\emph{Pooled by backbone (across 4 tasks)}}\\
Granite 3.3-2B & 84.6$_{[80.6,88.0]}$ \\
Granite 3.3-8B & 84.6$_{[80.8,87.6]}$ \\
Granite 4.0-1B & 86.0$_{[82.4,88.9]}$ \\
\midrule
\multicolumn{2}{l}{\emph{Pooled by task (across 3 backbones)}}\\
LibriSQA               & 93.2$_{[90.9,94.7]}$ \\
DailyTalk    & 91.6$_{[88.1,94.1]}$ \\
CoVoST2           & 81.1$_{[75.9,85.1]}$ \\
LongAudio       & 74.0$_{[70.1,77.6]}$ \\
\midrule
All 12 cells           & \textbf{85.1}$_{[82.9,86.9]}$ \\
\bottomrule\end{tabular}
\label{tab:fork-headline}
\end{table}

\begin{figure}[t]\centering
\includegraphics[width=\linewidth]{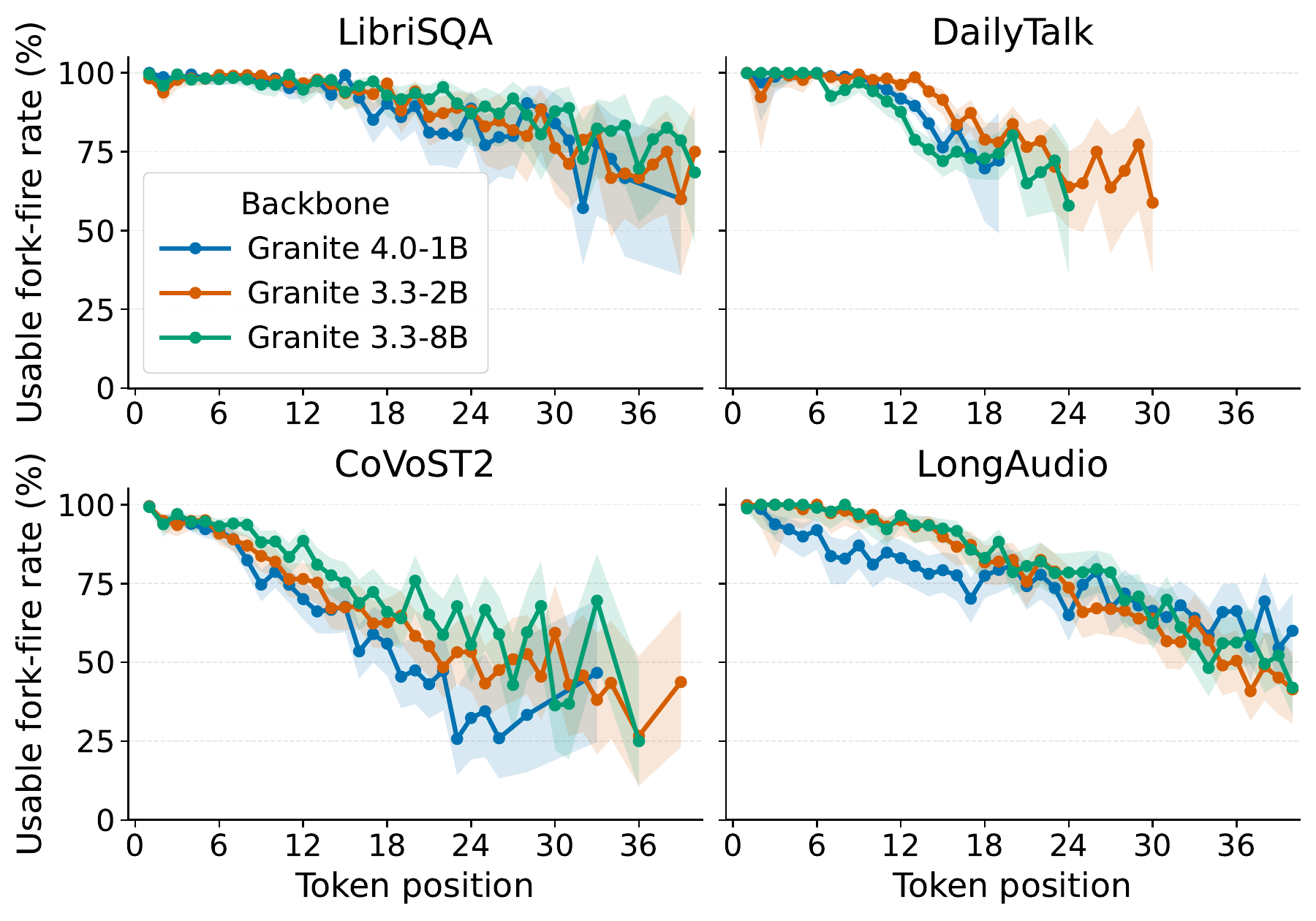}
\caption{Usable fork-fire rate vs.\ token position, one panel per task,
three backbones overlaid. Shaded bands are pointwise $95\%$ CIs;
positions with fewer than $15$ attempts are dropped to suppress
small-$n$ noise; $x$-axis is clipped to position $40$. Backbones overlap
within their CIs in every panel, supporting backbone-invariance.}
\label{fig:fork-fire-position}
\end{figure}

\textbf{Forks branch in distinct outcomes, and reach the
response.} A high rate alone is not enough: fires must produce
non-singleton groups with discriminating rewards, and must occur deep
enough into the response to discriminate body tokens rather than just
the opening. Per-cell statistics in the appendix
(Tab.~\ref{tab:fork-aux-app}, \ref{tab:fork-decay-app}) confirm both:
mean group sizes lie in $4.0\text{--}6.2$ of the $K\!=\!8$ budget,
within-group BLEU spread is strictly positive in every cell, and the
$90\%$-cumulative depth $P_{90}$ ranges from $14$ to $39$ token
positions across cells. Forks reach into the body, particularly on
long-form QA where $P_{90}\!\ge\!35$ across all three
backbones.

\section{Results}
\label{sec:results}

\begin{table*}[t]
\centering\scriptsize
\caption{End-to-end results across datasets, backbones, and training methods. Each row is the best-val BLEU checkpoint. M-Prometheus Likert is a 1--5 rubric-grounded judge; GEMBA-DA is a 0--100 continuous score.}
\label{tab:master}

\begin{tabular}{@{}l@{\hspace{3pt}}|l@{\hspace{3pt}}l|@{\hspace{3pt}}rrrrrrrrr@{}}
\toprule
\multicolumn{1}{l}{Dataset} &
\multicolumn{1}{l}{Backbone} &
\multicolumn{1}{l}{Method} &
BLEU & R-1 & R-2 & R-L & MTR & BS-F1 & Toks & Likert$_{1\text{-}5}$ [CI] & DA{$_{0\text{-}100}$} [CI] \\
\midrule

\multirow{12}{*}{LibriSQA}
& \multirow{4}{*}{Granite 3 2B}
& base & 28.02 & 56.80 & 40.35 & 51.39 & 53.02 & 91.18 & 22.3 & 2.65\,$_{[2.59,\,2.71]}$ & 70.52\,$_{[69.27,\,71.71]}$ \\
& & +Full GRPO & 44.90 & 68.56 & 53.35 & 64.88 & 68.48 & 94.45 & --- & --- & --- \\
& & +LoRA GRPO & 40.73 & 67.25 & 52.33 & 63.76 & 66.78 & 94.01 & 19.7 & 3.20\,$_{[3.15,\,3.26]}$ & 84.89\,$_{[84.00,\,85.72]}$ \\
& & \best{+LEAF} & \best{45.82} & \best{68.84} & \best{54.19} & \best{65.60} & \best{68.53} & \best{94.64} & \best{18.5} & \best{3.34\,$_{[3.29,\,3.39]}$} & \best{86.64\,$_{[85.83,\,87.38]}$} \\
\cmidrule{2-12}

& \multirow{3}{*}{Granite 3 8B}
& base & 24.62 & 54.33 & 38.34 & 48.43 & 53.29 & 90.87 & 26.1 & 2.28\,$_{[2.22,\,2.33]}$ & 64.57\,$_{[63.26,\,65.81]}$ \\
& & +Full GRPO & 46.40 & 69.57 & 57.49 & 66.16 & 69.61 & 94.76 & --- & --- & --- \\
& & \best{+LEAF} & \best{47.33} & \best{69.65} & \best{57.69} & \best{66.42} & \best{69.71} & \best{94.80} & \best{18.6} & \best{3.41\,$_{[3.35,\,3.46]}$} & \best{87.67\,$_{[86.95,\,88.45]}$} \\
\cmidrule{2-12}

& \multirow{2}{*}{Granite 4.0 1B}
& base & 27.78 & 56.43 & 40.00 & 51.00 & 52.97 & 91.13 & 22.5 & 2.62\,$_{[2.56,\,2.68]}$ & 69.99\,$_{[68.71,\,71.18]}$ \\
& & \best{+LEAF} & \best{44.61} & \best{67.65} & \best{52.60} & \best{64.11} & \best{67.37} & \best{94.36} & \best{18.6} & \best{3.21\,$_{[3.15,\,3.27]}$} & \best{84.38\,$_{[83.44,\,85.30]}$} \\
\cmidrule{2-12}

& \multirow{3}{*}{Qwen2-Audio}
& base & 12.93 & 38.14 & 24.99 & 34.59 & 40.12 & 88.19 & 45.0 & 2.09\,$_{[2.03,\,2.15]}$ & 55.74\,$_{[54.23,\,57.25]}$ \\
& & +LoRA GRPO & 40.95 & 66.42 & 51.59 & 63.05 & 66.89 & 94.06 & 20.1 & 3.32\,$_{[3.27,\,3.38]}$ & 83.62\,$_{[82.63,\,84.48]}$ \\
& & \best{+LEAF} & \best{45.85} & \best{68.94} & \best{54.32} & \best{65.59} & \best{68.65} & \best{94.64} & \best{18.8} & \best{3.43\,$_{[3.37,\,3.48]}$} & \best{85.78\,$_{[84.94,\,86.58]}$} \\

\midrule

\multirow{3}{*}{DailyTalk}
& \multirow{3}{*}{Granite 3 2B}
& base & 2.00 & 15.52 & 4.79 & 12.27 & 22.94 & 84.04 & 71.8 & 1.30\,$_{[1.28,\,1.33]}$ & 44.04\,$_{[43.13,\,44.96]}$ \\
& & +LoRA GRPO & 17.56 & 46.91 & 27.97 & 43.35 & 43.70 & 91.63 & 19.1 & 2.56\,$_{[2.52,\,2.60]}$ & 62.62\,$_{[61.60,\,63.66]}$ \\
& & \best{+LEAF} & \best{22.61} & \best{51.98} & \best{31.14} & \best{48.14} & \best{48.46} & \best{92.49} & \best{15.3} & \best{2.86\,$_{[2.82,\,2.90]}$} & \best{68.99\,$_{[68.11,\,69.88]}$} \\

\midrule

\multirow{3}{*}{LongAudio}
& \multirow{3}{*}{Granite 3 2B}
& base & 8.46 & 31.01 & 13.91 & 23.67 & 34.76 & 86.70 & 70.1 & 2.01\,$_{[1.94,\,2.07]}$ & 66.61\,$_{[65.21,\,67.96]}$ \\
& & +LoRA GRPO & 22.44 & 49.60 & 31.77 & 43.56 & 45.59 & 90.83 & 40.5 & 2.43\,$_{[2.37,\,2.49]}$ & 67.11\,$_{[65.79,\,68.46]}$ \\
& & \best{+LEAF} & \best{31.12} & \best{56.24} & \best{37.56} & \best{50.27} & \best{50.62} & \best{92.63} & \best{30.2} & \best{2.93\,$_{[2.88,\,2.97]}$} & \best{70.21\,$_{[69.14,\,71.30]}$} \\

\midrule

\multirow{8}{*}{CoVoST2}
& \multirow{3}{*}{Granite 3 2B}
& base & 11.45 & 23.34 & 12.88 & 22.16 & 25.44 & 79.78 & 15.5 & 1.08\,$_{[1.07,\,1.08]}$ & 45.07\,$_{[44.69,\,45.47]}$ \\
& & +LoRA GRPO & 27.77 & 55.68 & 33.76 & 52.59 & 53.28 & 85.43 & 22.1 & 1.92\,$_{[1.90,\,1.94]}$ & 64.94\,$_{[64.47,\,65.38]}$ \\
& & \best{+LEAF} & \best{30.57} & \best{59.20} & \best{37.21} & \best{56.04} & \best{56.14} & \best{86.28} & \best{21.5} & \best{2.13\,$_{[2.11,\,2.16]}$} & \best{69.09\,$_{[68.64,\,69.54]}$} \\
\cmidrule{2-12}

& \multirow{3}{*}{Granite 3 8B}
& base & 16.66 & 30.55 & 18.30 & 28.93 & 31.98 & 81.16 & 17.1 & 1.14\,$_{[1.14,\,1.15]}$ & 49.32\,$_{[48.94,\,49.72]}$ \\
& & +LoRA GRPO & 31.95 & 59.79 & 38.01 & 56.48 & 57.41 & 86.70 & 22.3 & 2.25\,$_{[2.23,\,2.27]}$ & 73.05\,$_{[72.64,\,73.47]}$ \\
& & \best{+LEAF} & \best{33.87} & \best{61.58} & \best{40.09} & \best{58.31} & \best{59.11} & \best{87.20} & \best{21.9} & \best{2.42\,$_{[2.40,\,2.45]}$} & \best{75.01\,$_{[74.59,\,75.43]}$} \\
\cmidrule{2-12}

& \multirow{2}{*}{Granite 4.0 1B}
& base & 11.80 & 23.67 & 13.15 & 22.51 & 25.91 & 79.91 & 15.5 & 1.07\,$_{[1.07,\,1.08]}$ & 44.78\,$_{[44.42,\,45.15]}$ \\
& & \best{+LEAF} & \best{31.12} & \best{59.04} & \best{37.12} & \best{55.84} & \best{56.74} & \best{86.44} & \best{21.9} & \best{2.14\,$_{[2.11,\,2.16]}$} & \best{68.39\,$_{[67.92,\,68.83]}$} \\

\bottomrule
\end{tabular}
\end{table*}

\begin{table*}[t]
\centering\scriptsize
\caption{Tail-risk analysis on GEMBA-DA-100 continuous scores. CVaR@$\alpha$ is the mean of the worst $\alpha$ fraction of items; higher is better, except frac$<50$ where lower is better. Bootstrap 95\% CIs use $1000$ resamples.}
\label{tab:cvar}

\begin{tabular}{@{}l|ll|rrrrr@{}}
\toprule
\multicolumn{1}{l}{Dataset} &
\multicolumn{1}{l}{Backbone} &
\multicolumn{1}{l}{Method} &
$N$ & Mean [95\% CI] & CVaR@10 [CI] & CVaR@25 [CI] & frac~$<$~50\\
\midrule
\multirow{10}{*}{LibriSQA}
& \multirow{3}{*}{Qwen2-Audio}
& base & 2620 & 55.74 [54.23, 57.25] & 0.00 [0.00, 0.00] & 1.25 [0.80, 1.88] & 38.59\% \\
& & +LoRA GRPO & 2620 & 83.62 [82.63, 84.48] & 19.06 [15.82, 23.40] & 51.04 [47.74, 54.08] & 10.31\% \\
& & \best{+LEAF} & 2620 & \best{85.78 [84.94, 86.58]} & \best{29.31 [24.82, 33.73]} & \best{58.19 [55.46, 61.01]} & \best{7.86\%} \\
\cmidrule{2-8}
& \multirow{3}{*}{Granite 3 2B}
& base & 2620 & 70.52 [69.27, 71.71] & 6.25 [5.18, 7.60] & 22.12 [19.98, 24.25] & 24.77\% \\
& & +LoRA GRPO & 2620 & 84.89 [84.00, 85.72] & 24.27 [20.76, 28.27] & 54.04 [50.99, 56.91] & 9.62\%  \\
& & \best{+LEAF} & 2620 & \best{86.64 [85.83, 87.38]} & \best{32.04 [28.03, 36.30]} & \best{59.56 [56.81, 62.06]} & \best{7.40\% } \\
\cmidrule{2-8}
& \multirow{2}{*}{Granite 3 8B}
& base & 2620 & 64.57 [63.26, 65.81] & 5.15 [4.41, 6.06] & 16.96 [15.34, 18.63] & 31.56\% \\
& & \best{+LEAF} & 2620 & \best{87.67 [86.95, 88.45]} & \best{36.04 [31.61, 40.66]} & \best{62.78 [60.14, 65.49]} & \best{6.45\%} \\
\cmidrule{2-8}
& \multirow{2}{*}{Granite 4.0 1B}
& base & 2620 & 69.99 [68.71, 71.18] & 5.85 [4.80, 7.13] & 21.57 [19.42, 23.64] & 25.61\% \\
& & \best{+LEAF} & 2620 & \best{84.38 [83.44, 85.30]} & \best{21.15 [17.60, 25.38]} & \best{52.38 [49.39, 55.18]} & \best{9.77\%} \\
\midrule
\multirow{3}{*}{DailyTalk}
& \multirow{3}{*}{Granite 3 2B}
& base & 3150 & 44.04 [43.13, 44.96] & 2.05 [1.68, 2.49] & 8.23 [7.40, 9.09] & 52.70\% \\
& & +LoRA GRPO & 3150 & 62.62 [61.60, 63.66] & 6.92 [6.09, 7.82] & 19.74 [18.31, 21.28] & 30.86\%\\
& & \best{+LEAF} & 3150 & \best{68.99 [68.11, 69.88]} & \best{11.86 [10.38, 13.51]} & \best{30.44 [28.49, 32.51]} & \best{20.73\%} \\
\midrule
\multirow{3}{*}{LongAudio}
& \multirow{3}{*}{Granite 3 2B}
& base & 1632 & 66.61 [65.21, 67.96] & 1.26 [0.79, 2.13] & 20.65 [17.08, 24.31] & 21.57\%\\
& & +LoRA GRPO & 1632 & 67.11 [65.79, 68.46] & 8.89 [7.15, 11.00] & 26.88 [24.23, 29.82] & 22.61\% \\
& & \best{+LEAF} & 1632 & \best{70.21 [69.14, 71.30]} & \best{16.62 [13.96, 19.47]} & \best{36.31 [33.69, 39.17]} & \best{16.91\%} \\
\midrule
\multirow{8}{*}{CoVoST2}
& \multirow{3}{*}{Granite 3 2B}
& base & 15485 & 45.07 [44.69, 45.47] & 5.81 [5.47, 6.14] & 13.18 [12.80, 13.58] & 57.02\% \\
& & +LoRA GRPO & 15485 & 64.94 [64.47, 65.38] & 10.65 [10.25, 11.10] & 21.53 [20.93, 22.08] & 31.35\% \\
& & \best{+LEAF} & 15485 & \best{69.09 [68.64, 69.54]} & \best{12.27 [11.77, 12.81]} & \best{25.34 [24.62, 26.11]} & \best{25.98\% } \\
\cmidrule{2-8}
& \multirow{3}{*}{Granite 3 8B}
& base & 15485 & 49.32 [48.94, 49.72] & 6.42 [6.05, 6.81] & 14.73 [14.31, 15.17] & 49.17\% \\
& & +LoRA GRPO & 15485 & 73.05 [72.64, 73.47] & 14.97 [14.27, 15.67] & 31.64 [30.73, 32.53] & 20.89\% \\
& & \best{+LEAF} & 15485 & \best{75.01 [74.59, 75.43]} & \best{15.08 [13.82, 15.29]} & \best{34.02 [32.96, 35.17]} & \best{18.66\%} \\
\cmidrule{2-8}
& \multirow{2}{*}{Granite 4.0 1B}
& base & 15485 & 44.78 [44.42, 45.15] & 7.02 [6.67, 7.40] & 14.28 [13.95, 14.70] & 58.22\%\\
& & \best{+LEAF} & 15485 & \best{68.39 [67.92, 68.83]} & \best{12.24 [11.72, 12.76]} & \best{24.70 [24.04, 25.41]} & \best{26.99\%} \\
\bottomrule
\end{tabular}
\end{table*}

\begin{table*}[t]
\centering\scriptsize
\caption{Pairwise A/B evaluation of LEAF vs.\ GRPO with position-swap consistency aggregation. Consistency, LEAF win, GRPO win, and tie are reported as fractions of all evaluated pairs; LEAF win (decided) is computed only among consistent non-tie pairs.}
\label{tab:pairwise}
\begin{tabular}{@{}lllrrrrrrr@{}}
\toprule
Judge & Dataset & Backbone & $N$ & Consist. & LEAF win & GRPO win & Tie & Net adv. & LEAF win (decided) \\
\midrule
\multirow{4}{*}{JudgeLM}
& \multirow{2}{*}{LibriSQA}
& Granite 3 2B
& 2620 & 80.0\% & \best{11.4\%} & 3.8\% & 64.8\% & \best{+7.6\%} & \best{75.1\%} \\

& & Qwen2-Audio
& 2620 & 82.0\% & \best{7.8\%} & 3.9\% & 70.3\% & \best{+3.9\%} & \best{66.6\%} \\

& DailyTalk
& Granite 3 2B
& 3150 & 68.2\% & \best{24.0\%} & 12.6\% & 31.6\% & \best{+11.4\%} & \best{65.5\%} \\

& LongAudio
& Granite 3 2B
& 1632 & 57.9\% & \best{30.5\%} & 9.3\% & 18.1\% & \best{+21.2\%} & \best{76.7\%} \\

\midrule
\multirow{2}{*}{M-Prom.}
& \multirow{2}{*}{CoVoST2}
& Granite 3 2B
& 15485 & 60.7\% & \best{39.2\%} & 21.5\% & 0.0\% & \best{+17.7\%} & \best{64.6\%} \\

& & Granite 3 8B
& 15485 & 54.8\% & \best{34.4\%} & 20.4\% & 0.0\% & \best{+14.0\%} & \best{62.8\%} \\
\bottomrule
\end{tabular}
\end{table*}

We organize the results around two comparisons. First, we compare LEAF against the published speech-RL state of the art on LibriSQA, where the evaluation setup is closely reproducible. Second, we compare LEAF against our own LoRA GRPO baseline across all datasets, which controls for prompts, decoding, rollout budget, trainable-parameter budget, and evaluation pipeline.

\subsection{Comparison Against the State-of-the-art}
\label{subsec:sota}

Table~\ref{tab:master} includes the published full-parameter GRPO results of \citet{elmakies2025speechgrpo} on LibriSQA. Despite using only LoRA adapters, LEAF improves over their full-parameter GRPO numbers on the matched Granite Speech backbones. On Granite~3.3~2B, LEAF improves over full-parameter GRPO while training only $133.7$M parameters instead of $3.01$B, a $22.5\times$ reduction in trainable parameters. On Granite~3.3~8B, LEAF trains $218.9$M parameters instead of $8.65$B, a $39.5\times$ reduction, while again improving the reported LibriSQA performance. The smaller Granite-4.0-1B-Speech model trained with LEAF uses only $102.2$M trainable parameters and approaches the published full-parameter Granite~3.3~2B GRPO result, further suggesting that improved credit assignment can compensate for a substantially smaller trainable-parameter budget.

We restrict cross-paper comparison to LibriSQA. For CoVoST2, the frozen-base Granite~3.3~2B/8B results reported by \citet{elmakies2025speechgrpo} are substantially higher than the frozen-base performance obtained under our prompts and input formatting. This mismatch suggests that the CoVoST2 setups differ in prompt template or input formatting, not only in the RL algorithm. Thereupon, we treat our LoRA GRPO implementation as the controlled CoVoST2 baseline and do not include the published CoVoST2 numbers, and we provide all prompts that were used in our experiments in the Appendix.

\subsection{Controlled Comparison Against GRPO}
\label{subsec:controlled-grpo}

Table~\ref{tab:master} gives the controlled comparison between LEAF and LoRA GRPO across all datasets and backbones. LEAF improves over LoRA GRPO in every matched setting, with BLEU gains ranging from $+1.92$ to $+8.68$. The largest gains occur on the long-form QA settings, where broadcasting a single terminal-reward advantage across the entire response is most likely to dilute span-level signals.

The judge-based metrics show the same trend. LEAF improves GEMBA-DA in every matched cell, with gains ranging from $+1.75$ to $+6.37$, and likewise improves Likert scores throughout. Thus, the gains are not limited to n-gram overlap metrics, but persist under both rubric-grounded and continuous LLM-as-judge evaluation.

\noindent\textbf{Tail-risk.} Table~\ref{tab:cvar} shows that LEAF improves worst-case behavior in addition to average quality. Across the matched QA settings, LEAF improves CVaR@10 in every case, with gains ranging from $+4.94$ to $+10.25$. On CoVoST2, LEAF likewise improves CVaR@25 and reduces the fraction of examples scoring below $50$ for both Granite backbones. These results suggest that LEAF reduces severe failures, consistent with localizing negative credit to harmful spans rather than broadcasting it across the entire response.

\noindent\textbf{Pairwise preference.}
Table~\ref{tab:pairwise} reports direct A/B comparisons between LEAF and GRPO on matched examples. LEAF achieves a positive net win advantage in every evaluated setting, ranging from $+3.9$ to $+21.2$ percentage points across JudgeLM and M-Prometheus evaluations. Thus, the improvements in judge scores are also reflected in direct output-level preferences for LEAF.

\noindent\textbf{Length and runtime efficiency.}
LEAF produces more compact outputs than GRPO, reducing predicted tokens by $6.1-6.5\%$ on LibriSQA, and by $19.9\%$ and $25.4\%$ on DailyTalk and LongAudio. Importantly, LEAF reuses the same sampled completions as GRPO and requires no additional decoding, model forward passes, process reward model, or architectural changes, adding only vectorized batch post-processing over at most $B$ boundaries. In our Granite~2B run, LEAF takes $8.36$s/step compared with $12.52$s/step for GRPO, a $33.2\%$ reduction consistent with its shorter generations.

\section{Ablation Studies}
\label{sec:ablations}

We now examine the design choices underlying LEAF and shed light into our theoretical insights empirically. All ablations use Granite Speech~3.3~2B trained on LibriSQA for a single epoch, with the same hyperparameters as our other experiments. We use a single training seed for these controlled experiments. For the first two ablations, we report BLEU, DA-100, and Likert scores, and summarize performance using a normalized average (Norm.\ Avg.) computed as
$
\mathrm{Norm.\ Avg.}
=
(\mathrm{BLEU}+\mathrm{DA\text{-}100}+25(\mathrm{Likert}-1))/3.
$

\subsection{Effect of the Fork Budget}

Regarding the fork budget, our design uses $B=2$ as a natural operating point. With only one retained boundary along a response path ($B=1$), the global and local advantages coincide.
Thus, a single boundary cannot provide a genuinely distinct parent-relative signal in addition to the root-relative signal. A second retained boundary is the smallest setting in which the two can differ, making $B=2$ the minimum fork budget needed to realize LEAF’s credit decomposition. At the same time, Theorem ~\ref{thrm:fork_budget_informal} highlights a trade-off: increasing \(B\) improves credit-assignment resolution but can introduce more low-support prefix-value estimates. We therefore examine whether increasing $B\geq2$ is beneficial.

Table~\ref{tab:fork_budget_ablation} shows that $B=2$ performs best across all three metrics and obtains the highest normalized average. Increasing $B$ does not monotonically improve performance: although additional boundaries permit a finer decomposition of the response, all larger budgets remain below $B=2$. This behavior is consistent with the support--resolution tradeoff in Theorem~\ref{thrm:fork_budget_informal}.

\begin{table}[t]
\centering
\caption{Fork-budget ablation on LibriSQA with Granite Speech~3.3~2B with one training epoch. The trend in performance is non-monotone; $B=2$ performs best.}
\small
\begin{tabular}{lrrrr}
\toprule
$B$ & BLEU & DA-100 & Likert & Norm.\ Avg.\\
\midrule
\bf{2} & \bf{43.91} & \bf{85.47} & \bf{3.23} & \bf{61.71}\\
4 & 39.37 & 72.73 & 2.84 & 52.70\\
6 & 40.45 & 78.69 & 3.00 & 56.38\\
8 & 41.22 & 81.12 & 3.07 & 58.03\\
\bottomrule
\end{tabular}
\label{tab:fork_budget_ablation}
\end{table}

\subsection{Effect of the Advantage Form}

LEAF combines two complementary prefix-level signals, $\mathrm{GA}$ and $\mathrm{LA}$. Theorem~\ref{thrm:adv_informal} establishes that both are individually valid span-level credit signals. Hence, we ablate the estimator by retaining only $\mathrm{GA}$, only $\mathrm{LA}$, or their combination. As shown in Table~\ref{tab:advantage_ablation}, $\mathrm{GA}$ and $\mathrm{LA}$ achieve very similar performance when used independently, consistent with Theorem~\ref{thrm:adv_informal}. Their combination performs best across all metrics, reflecting complementary information: GA evaluates a prefix relative to the rollout group, while LA captures its incremental change from the preceding prefix.

\begin{table}[t]
\centering
\caption{Advantage-form ablation on LibriSQA with Granite Speech~3.3~2B with one training epoch. Local and global advantage are both valid forms with similar performance, while their sum performs best.}
\small
\begin{tabular}{lrrrr}
\toprule
Advantage & BLEU & DA-100 & Likert & Norm.\ Avg.\\
\midrule
GA & 43.24 & 84.46 & 3.19 & 60.82\\
LA & 43.32 & 84.65 & 3.20 & 60.99\\
\bf{GA+LA} & \bf{43.91} & \bf{85.47} & \bf{3.23} & \bf{61.71}\\
\bottomrule
\end{tabular}
\label{tab:advantage_ablation}
\end{table}

\subsection{Cross-Benchmark Generalization}

We examine whether the improvements obtained by LEAF are specific to the LibriSQA data used for training. We evaluate the same one-epoch Granite Speech~3.3~2B GRPO and LEAF checkpoints, on the AlpacaEval, CommonEval, and WildVoice subsets of VoiceBench \cite{chen2024voicebench} and we use M-Prometheus-14B, due to its competitive judging performance \citep{pombal2025mprometheus}, to assign $1$--$5$ quality scores to the generated responses. Figure~\ref{fig:voicebench_generalization} shows that LEAF improves over GRPO on all three unseen benchmarks, and the average score increases from $2.373$ to $3.352$. These results therefore indicate that LEAF's gains transfer to unseen open-ended speech settings, demonstrating that its finer span-level credit assignment improves more general speech-conditioned response behavior.

\begin{figure}[t]\centering
\includegraphics[width=0.9\linewidth]{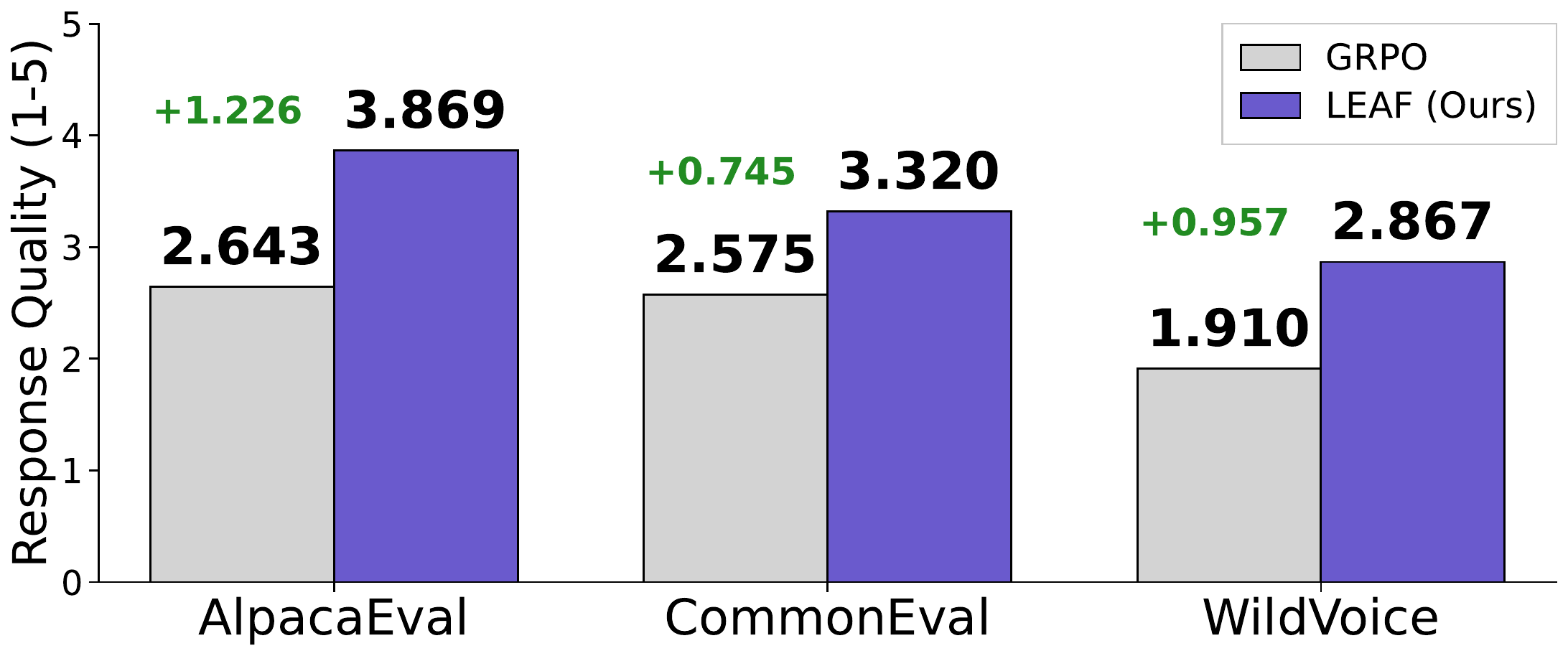}
\caption{Cross-benchmark evaluation on three open-ended VoiceBench subsets using M-Prometheus-14B as the judge. Both checkpoints are trained only on LibriSQA with Granite Speech~3.3~2B for one epoch. LEAF outperforms GRPO in all cases.}
\label{fig:voicebench_generalization}
\end{figure}

\section{Conclusion}
In this work, we introduced LEAF, a retrospective tree-based RL method for SALLM post-training. LEAF reuses the same i.i.d.\ rollout groups collected by GRPO, detects shared prefixes at high-surprisal boundaries, and converts descendant rewards into span-level advantages. This yields process-aware credit assignment without extra decoding, online tree search, or a separate process reward model. Across speech QA and translation tasks, LEAF consistently improves over LoRA GRPO under the same rollout and trainable-parameter budget, while also improving tail behavior, pairwise preferences, and length and runtime efficiency. Importantly, on LibriSQA, LEAF further improves over SOTA full-parameter GRPO while training only LoRA adapters. Future work includes richer reward models, multilingual speech tasks, and adaptive boundary selection.

\section*{Limitations}
We note several limitations. First, LEAF's effectiveness is regime-dependent: its mechanism requires that i.i.d.\ rollouts share prefixes often enough to populate non-trivial fork groups, and the absolute fork rate depends on the policy. We therefore state our claims comparatively (backbone-invariance, task-monotonicity) and as a lower bound ($r_f \ge 72.9\%$ in every cell) rather than as an exact rate. Second, our reward signal is BLEU score, following prior speech-RL work; smoother rewards may yield further gains we do not measure. Third, hardware constraints limited us to three NVIDIA A40 GPUs (48~GB) with a single training seed per
(model, method) and thus do not characterize training-seed variance. Moreover, our ablations also used a single seed, and therefore interpret them primarily as evidence for the qualitative behavior of the proposed design choices. Finally, all speech inputs are English, so our findings may not transfer to non-English source speech or other modalities.

\section*{Potential Risks and Societal Impact}
LEAF is a post-training method rather than a deployed speech system, but improving speech understanding can amplify risks associated with downstream SALLM use. First, model errors can misrepresent the content or intent of source audio and propagate into user-facing decisions; our benchmark improvements should therefore not be interpreted as establishing reliability in high-stakes settings. Second, we evaluate only English source speech and do not study performance across accents, dialects, demographic groups, or non-English languages. Consequently, downstream systems may exhibit uneven error rates for underrepresented speakers. Third, stronger speech understanding is inherently dual-use: large-scale analysis of recorded speech could facilitate unwanted monitoring or surveillance and may expose sensitive information contained in speech. Finally, multimodal speech training carries non-trivial compute and energy costs. LEAF does not require additional rollouts or online tree expansion beyond the fixed GRPO rollout batch, limiting its incremental generation cost relative to GRPO, but training and evaluating the underlying SALLMs remains computationally expensive. Deployment-oriented work should therefore evaluate privacy, demographic robustness, and application-specific safety before use in consequential settings.

\section*{AI Usage}
AI-based assistants were used for limited editing, wording suggestions, and debugging assistance during experiment development and manuscript preparation. All technical content, experimental design, analysis, and conclusions were verified and finalized by the authors.

\section*{Acknowledgements}
AG gratefully acknowledges Moulik Choraria for introducing him to tree-based methods for post-training large language models. This work was supported in part by a grant from the Institute for Information \& Communications Technology Planning \& Evaluation (IITP) to KAIST and in part by the National Science Foundation under grant \#2106727. The authors also acknowledge the Advanced Cyberinfrastructure Coordination Ecosystem: Services \& Support (ACCESS)  program~\cite{boerner2023access} for providing computing allocations and support through allocation CIS260224 (NSF grants \#2138259, \#2138286, \#2138307, \#2137603, and \#2138296).

\bibliography{main}
\appendix

\section{Hardware details}\label{app:train}

All experiments were run using three NVIDIA A40 GPUs (48~GB).

\section{Experimental Details}

\subsection{Training Parameters}
\label{app:training_params}

\begin{table}[h]
\centering
\caption{Shared RL training hyperparameters. Unless otherwise stated, the same settings are used for LEAF and GRPO.}
\small
\begin{tabular}{ll}
\toprule
Hyperparameter & Value \\
\midrule
LoRA rank / alpha / dropout & $64$ / $128$ / $0.05$ \\
Optimizer & Adam \\
Learning rate & $5\times 10^{-6}$ \\
KL coefficient & $0.02$ \\
Rollout budget $K$ & $8$ responses per prompt \\
LEAF fork budget $B$ & $2$ \\
Rollout batch size & $6$ \\
Training batch size & $6$ \\
Micro-batch size & $1$ \\
Max prompt length & $256$ tokens \\
Max response length & $200$ tokens \\
Sampling temperature & $1.0$ \\
Top-$p$ & $0.9$ \\
Reward & Sentence BLEU \\
\bottomrule
\end{tabular}
\label{tab:training_hparams}
\end{table}

\subsection{Dataset Specifics}
To create the LibriSQA and the CoVoST2 datasets we follow the process shared in the original works \cite{wang2020covost2,zhao2023librisqa}. In order to create the DailyTalk dataset, we use the prompts and data given in LongAudio \cite{ghost2025longaudio} and integrate it accordingly with the original data \cite{lee2023dailytalk}, while retaining only examples with audio durations of at most 40 seconds.

\label{app:datasets}
\begin{table}[h]
\centering\small
\caption{Dataset splits used in this work.}
\label{tab:dataset-sizes}
\begin{tabular}{@{}lrrr@{}}
\toprule
Dataset & Train & Validation & Test \\
\midrule
LibriSQA & 85{,}045 & 21{,}196 & 2{,}620 \\
DailyTalk& 21{,}891 & 2{,}806 & 3{,}150 \\
LongAudio & 13{,}374 & 1{,}743 & 1{,}632 \\
CoVoST2 & 288{,}069 & 15{,}471 & 15{,}485 \\
\bottomrule
\end{tabular}
\end{table}

\begin{table}[t]
\centering
\caption{Number of RL training epochs. All other RL hyperparameters are shared as in Table~\ref{tab:training_hparams}.}
\scriptsize
\begin{tabular}{lll}
\toprule
Model family & Dataset setting & Epochs \\
\midrule
Granite Speech 3.3 2B / 8B & CoVoST2, LibriSQA & $2$ \\
Granite Speech 3.3 2B  & DailyTalk, LongAudio & $6$ \\
Granite-4.0-1B-Speech & LibriSQA & $2$ \\
Granite-4.0-1B-Speech & CoVoST2 & $1$ \\
Qwen2-Audio & All evaluated datasets & $1$ \\
\bottomrule
\end{tabular}
\label{tab:training_epochs}
\end{table}

\subsection{Prompt templates}
\label{app:prompts}

We list the prompt templates used for each dataset. The literal placeholder
\texttt{<|audio|>} is rewritten at tokenization time to the
backbone-specific audio token: a single audio token for Granite Speech and
\texttt{<|audio\_bos|><|AUDIO|><|audio\_eos|>} for Qwen2-Audio.

\paragraph{LibriSQA.}
We use the dataset-provided question field and wrap it with one of the
following templates:
\begin{itemize}
    \item \texttt{Listen to the audio <|audio|> and answer the following question: \{question\}}
    \item \texttt{<|audio|> Please answer this question about the spoken passage: \{question\}}
    \item \texttt{You will hear speech in <|audio|>. Answer: \{question\}}
\end{itemize}

\paragraph{CoVoST2 English-to-German.}
We use translation-instruction templates such as:
\begin{itemize}
    \item \texttt{<|audio|> Listen to the speech and translate it into German.}
    \item \texttt{Translate the audio <|audio|> from English to German.}
    \item \texttt{You will hear English speech in <|audio|>. Provide the German translation.}
\end{itemize}

\paragraph{DailyTalk and LongAudio.}
For the LongAudio-derived datasets, prompts are inherited from the dataset
preparation pipeline. They follow two structural variants:
\begin{itemize}
    \item \texttt{<|audio|>\textbackslash n\{question\}}
    \item \texttt{\{question\}\textbackslash n<|audio|>}
\end{itemize}
where \texttt{\{question\}} is the dataset-provided open-ended question.

\subsection{LLM-as-judge protocols}
\label{app:judges}

\paragraph{Prediction staging.}
For each dataset, backbone, method, and checkpoint, we generate a single
stochastic prediction file using temperature $0.9$, top-$p=0.9$, and maximum
generation length $200$. All judge protocols consume the same staged
prediction files, so judges are never given access to additional samples.

\paragraph{Grounded Likert-5.}
We use \textsc{Unbabel/M-Prometheus-14B} as a rubric-grounded judge on a
$1$--$5$ Likert scale. The judge receives the question, reference answer,
candidate prediction, and a task-specific rubric. For QA tasks, the rubric
emphasizes correctness, groundedness, completeness, and hallucination. For
CoVoST2, the rubric emphasizes accuracy, fluency, and style. Decoding is
greedy with temperature $0$.

\paragraph{GEMBA-DA $0$--$100$.}
For continuous per-item scoring, we use a GEMBA-style direct-assessment
prompt with \textsc{Qwen2.5-14B-Instruct}. Each item is judged $N=5$ times
at temperature $0.7$, and we use the per-item mean score as the canonical
DA-100 value.

\paragraph{Pairwise judging.}
For pairwise LEAF-vs.-GRPO comparisons, we use
\textsc{BAAI/JudgeLM-13B-v1.0} for QA tasks and M-Prometheus relative
grading for CoVoST2. Each pair is evaluated in both candidate orders. A
verdict is counted only when the two swapped evaluations agree; otherwise,
the pair is marked inconsistent and excluded from the decided-pair win rate.

\section{Forkability: per-cell measurements}
\label{app:fork}

This appendix expands Section~\ref{sec:why} with (i) the per-cell
breakdown of the usable fork-fire rate, (ii) the auxiliary
fork-group statistics that justify treating fires as
discriminating branches rather than near-clones, (iii) the
cumulative-position depth quantiles, and (iv) a backbone-organised
view of the per-position curves.

\subsection{Per-cell usable fork-fire rate}

Table~\ref{tab:fork-percell-app} gives the $12$-cell breakdown that
underlies the pooled rows of Tab.~\ref{tab:fork-headline}. The minimum
cell rate is $72.9\%$ (Granite-3.3-8B on LongAudio) and the
maximum is $94.6\%$ (Granite-4.0-1B on LibriSQA). Within every task
the three backbone CIs overlap; between tasks, the
LibriSQA / DailyTalk / CoVoST2 / LongAudio ordering recovers
the task-monotonicity claim.

\begin{table}[h]
\centering
\caption{Per-cell usable fork-fire rate $r_f$ with $95\%$ per-position
bootstrap CIs ($n_{\text{boot}}\!=\!5000$). $N$ is the number of
validation+test examples measured per cell.}
\tiny
\setlength{\tabcolsep}{4pt}
\begin{tabular}{llrrrr}\toprule
Backbone & Task & $N$ & Selected & Usable & $r_f$ (\%, 95\% CI) \\\midrule
G.\,3.3-2B & LibriSQA            & 3400 & 6800 & 6222 & 91.5 \,(86.9--94.1) \\
G.\,3.3-2B & CoVoST2        & 3400 & 6799 & 5364 & 78.9 \,(70.5--85.4) \\
G.\,3.3-2B & DailyTalk & 3400 & 6800 & 6364 & 93.6 \,(86.9--96.6) \\
G.\,3.3-2B & LongAudio    & 3000 & 6000 & 4395 & 73.3 \,(66.3--79.3) \\
G.\,3.3-8B & LibriSQA            & 3400 & 6799 & 6352 & 93.4 \,(89.4--95.8) \\
G.\,3.3-8B & CoVoST2        & 3000 & 6000 & 5055 & 84.3 \,(76.1--89.8) \\
G.\,3.3-8B & DailyTalk & 3000 & 6000 & 5273 & 87.9 \,(81.4--93.1) \\
G.\,3.3-8B & LongAudio    & 3332 & 6664 & 4857 & 72.9 \,(66.3--78.6) \\
G.\,4.0-1B & LibriSQA            & 3400 & 6799 & 6429 & 94.6 \,(90.9--96.6) \\
G.\,4.0-1B & CoVoST2        & 3400 & 6790 & 5472 & 80.6 \,(69.7--87.3) \\
G.\,4.0-1B & DailyTalk & 3400 & 6800 & 6316 & 92.9 \,(87.8--96.1) \\
G.\,4.0-1B & LongAudio    & 3332 & 6664 & 5042 & 75.7 \,(68.8--81.7) \\
\bottomrule\end{tabular}
\label{tab:fork-percell-app}
\end{table}

\subsection{Fork-group statistics}

Table~\ref{tab:fork-aux-app} reports the mean number of rollouts
sharing each fired prefix (``Group size'', maximum $K\!=\!8$), the
mean within-group BLEU spread (``Reward spread''), the usable-count
weighted mean position of fired forks (``MU pos''), and the number of
distinct token positions hosting at least one usable fork (``\# pos'').
Group sizes range $4.0\text{--}6.2$ across cells and reward spread is
strictly positive in every cell ($\ge 0.027$), confirming that fired
forks discriminate higher- from lower-reward branches rather than
collapsing to duplicates.

\begin{table}[h]
\centering
\caption{Auxiliary per-cell fork-group statistics.}
\tiny
\setlength{\tabcolsep}{4pt}
\begin{tabular}{llrrrr}\toprule
Backbone & Task & Group size & Reward spread & MU pos & \#\,pos \\\midrule
G.\,3.3-2B & LibriSQA            & 5.40 & 0.062 & 12.4 & 99 \\
G.\,3.3-2B & CoVoST2        & 4.05 & 0.044 & 11.1 & 97 \\
G.\,3.3-2B & DailyTalk & 6.18 & 0.030 & ~9.1 & 70 \\
G.\,3.3-2B & LongAudio    & 4.31 & 0.036 & 18.5 & 78 \\
G.\,3.3-8B & LibriSQA            & 5.75 & 0.053 & 13.5 & 99 \\
G.\,3.3-8B & CoVoST2        & 4.19 & 0.050 & 10.0 & 93 \\
G.\,3.3-8B & DailyTalk & 4.77 & 0.040 & ~9.2 & 35 \\
G.\,3.3-8B & LongAudio    & 4.25 & 0.034 & 20.2 & 78 \\
G.\,4.0-1B & LibriSQA            & 5.48 & 0.059 & ~9.0 & 78 \\
G.\,4.0-1B & CoVoST2        & 4.73 & 0.057 & ~6.9 & 72 \\
G.\,4.0-1B & DailyTalk & 5.29 & 0.027 & ~8.2 & 53 \\
G.\,4.0-1B & LongAudio    & 3.96 & 0.030 & 17.0 & 77 \\
\bottomrule\end{tabular}
\label{tab:fork-aux-app}
\end{table}

\subsection{Depth profile of usable fork-fires}

Table~\ref{tab:fork-decay-app} reports cumulative-position quantiles:
$P_q$ is the token position by which $q\%$ of the cell's usable
fork-fires have already occurred. The depth profile clusters tightly
across backbones within a task, and hence, the forkable window is governed by
the task's input-output coupling, not by the model. Short-output QA
tasks concentrate fork-fires within the first $\approx\!18$ tokens
(DailyTalk: $P_{90}\le 18$ across all three backbones),
while long-form LongAudio maintains fork activity past position
$35$ ($P_{90}=35\text{--}39$, $P_{99}=55\text{--}59$).

\begin{table}[h]
\centering
\caption{Cumulative-position quantiles of usable fork-fires per cell.
The depth profile is essentially constant across backbones within a
task.}
\scriptsize
\setlength{\tabcolsep}{4pt}
\begin{tabular}{llrrrrr}\toprule
Backbone & Task & $P_{50}$ & $P_{90}$ & $P_{99}$ & Last & \# pos \\\midrule
G.\,3.3-2B & LibriSQA            &  8 & 28 & 84 & 99 & 99 \\
G.\,3.3-2B & CoVoST2        &  7 & 22 & 83 & 99 & 97 \\
G.\,3.3-2B & DailyTalk &  8 & 18 & 32 & 99 & 70 \\
G.\,3.3-2B & LongAudio    & 17 & 35 & 55 & 99 & 78 \\
G.\,3.3-8B & LibriSQA            &  7 & 33 & 90 & 99 & 99 \\
G.\,3.3-8B & CoVoST2        &  6 & 21 & 76 & 99 & 93 \\
G.\,3.3-8B & DailyTalk &  9 & 17 & 23 & 53 & 35 \\
G.\,3.3-8B & LongAudio    & 18 & 39 & 59 & 99 & 78 \\
G.\,4.0-1B & LibriSQA            &  6 & 20 & 41 & 99 & 78 \\
G.\,4.0-1B & CoVoST2        &  5 & 14 & 38 & 99 & 72 \\
G.\,4.0-1B & DailyTalk &  8 & 14 & 18 & 94 & 53 \\
G.\,4.0-1B & LongAudio    & 15 & 37 & 57 & 97 & 77 \\
\bottomrule\end{tabular}
\label{tab:fork-decay-app}
\end{table}

\subsection{Per-backbone view of the per-position curves}

Figure~\ref{fig:fork-fire-position-app} replots
Fig.~\ref{fig:fork-fire-position} with the three backbones placed in
separate panels and the four tasks overlaid in each, making the
between-task separation visually direct. The four task curves order
$\text{LibriSQA}>\text{DailyTalk}>\text{CoVoST2}>
\text{LongAudio}$ at every position beyond the startup region,
on every backbone, and the visual analogue of the pooled-by-task ordering
is in Tab.~\ref{tab:fork-headline}.

\begin{figure*}[h]\centering
\includegraphics[width=\linewidth]{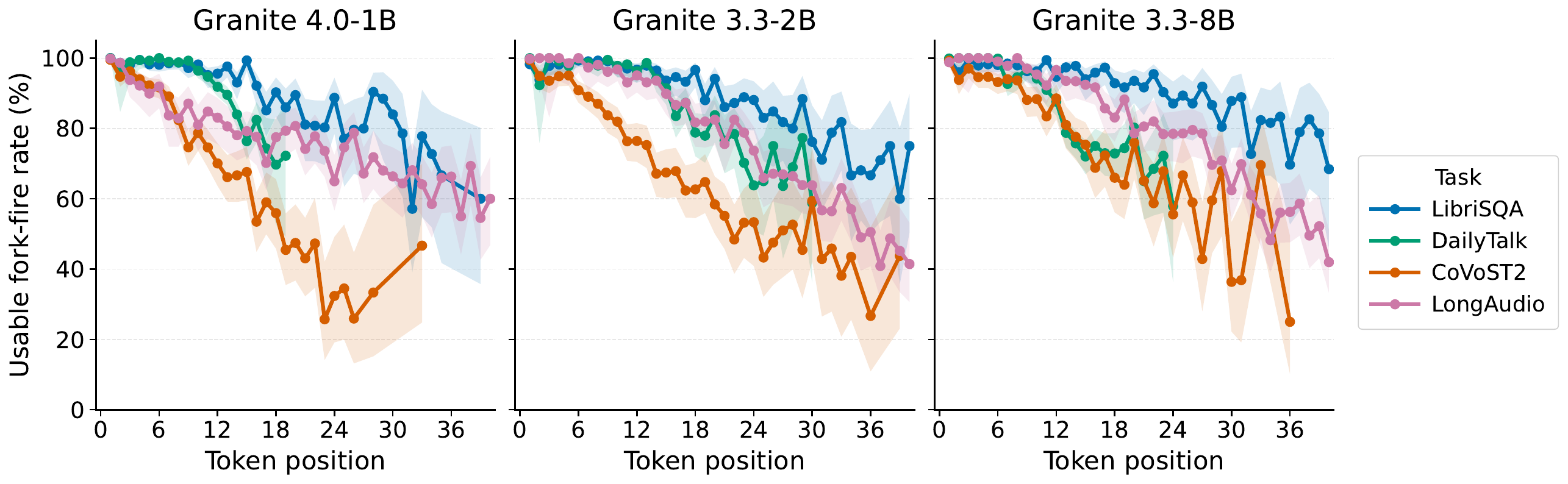}
\caption{Per-position usable fork-fire rate, one panel per backbone,
four tasks overlaid; shaded bands are pointwise $95\%$ CIs.}
\label{fig:fork-fire-position-app}
\end{figure*}

\section{Additional Experimental Results}
\label{app:additional_exps}
\subsection{Length Experiments}
\begin{table}[h]
\centering\scriptsize
\caption{Length efficiency: mean GPT-2 BPE token counts of the prediction (post-hoc, universal tokenizer; same for every row).}
\label{tab:length}
\begin{tabular}{@{}llrrr@{}}
\toprule
Dataset & Backbone & base & GRPO & LEAF \\
\midrule
CoVoST2 & Granite 3 2B & 15.5 & 22.1 & 21.5  \\
CoVoST2 & Granite 3 8B & 17.1 & 22.3 & 21.9  \\
\midrule
DailyTalk & Granite 3 2B & 71.8 & 19.1 & 15.3 \\
\midrule
LibriSQA & Granite 3 2B & 22.3 & 19.7 & 18.5  \\
LibriSQA & Qwen2-Audio & 45.0 & 20.1 & 18.8  \\
\midrule
LongAudio & Granite 3 2B & 70.1 & 40.5 & 30.2 \\
\bottomrule
\end{tabular}
\end{table}

\section{Theoretical Proofs}
\label{app:theory-proofs}

This appendix gives the formal statements and proofs for the theoretical claims in the main text. Throughout, we fix a prompt $x$. Let $Y\sim\pi_\theta(\cdot\mid x)$ be a random completion and let $R=r(x,Y)$ be its terminal reward. The observed rewards $r_i$ are realizations of $R$ for sampled completions $y^{(i)}$. We assume $R\in[0,R_{\max}]$ when using concentration inequalities.

\begin{assumption}[Fixed-prefix conditional sampling]
\label{ass:fixed-prefix-sampling}
Fix a boundary $p$ and a prefix $u\in\mathcal{V}^p$. Conditional on $n(p,u)=m$, the rewards of the sampled descendants $\{r_i:y_{<p}^{(i)}=u\}$ are conditionally independent samples from the distribution of $R$ given $Y_{<p}=u$. We write $V(u)=\mathbb{E}[R\mid Y_{<p}=u]$.
\end{assumption}

\begin{assumption}[Selected-node concentration]
\label{ass:selected-node-concentration}
For each retained node $v$, let $V_{\mathrm{sel}}(v)$ be its selection-conditional value. We assume that for all $\epsilon>0$,
\[
\mathbb{P}\left(|\widehat V(v)-V_{\mathrm{sel}}(v)|\ge\epsilon\right)
\le
2\exp\left(-\frac{2n(v)\epsilon^2}{R_{\max}^2}\right).
\]
\end{assumption}

\begin{remark}[Ordinary and selection-conditional prefix values]
Assumption~\ref{ass:fixed-prefix-sampling} is the clean condition for a fixed prefix. Since LEAF selects nodes retrospectively, a selected node can also carry information from the selection event. Assumption~\ref{ass:selected-node-concentration} therefore states concentration around $V_{\mathrm{sel}}(v)$, the value induced by the selected node. If selection does not distort the descendant reward law beyond conditioning on the retained prefix $u$, then $V_{\mathrm{sel}}(v)=V(u)$.
\end{remark}

\subsection{Proof of the span-advantage claim}

\begin{theorem}[Validity of LEAF span advantages]
\label{thm:leaf-advantage}
For a fixed prompt $x$, let $Y\sim\pi_\theta(\cdot\mid x)$, $R=r(x,Y)$, and $V(u)=\mathbb{E}[R\mid Y_{<|u|}=u]$. Consider a path of prefixes $U_0\prec U_1\prec\cdots\prec U_q$, with $U_0=\varnothing$, and let $Z_j=\sum_{t=\tau_{j-1}}^{\tau_j-1}\nabla_\theta\log\pi_\theta(A_t\mid S_t)$ be the score of span $j$. Define $\mathrm{GA}(U_j)=V(U_j)-V(U_0)$ and $\mathrm{LA}(U_j)=V(U_j)-V(U_{j-1})$. Then $\mathbb{E}[\mathrm{GA}(U_j)Z_j]=\mathbb{E}[\mathrm{LA}(U_j)Z_j]=\mathbb{E}[RZ_j]$. Moreover, if a retained node $v=(p,u)$ has $n(v)=m$ descendants, then under Assumption~\ref{ass:fixed-prefix-sampling}, $\mathbb{E}[\widehat V(v)\mid n(v)=m]=V(u)$ and $\operatorname{Var}(\widehat V(v)\mid n(v)=m)=\operatorname{Var}(R\mid Y_{<p}=u)/m$.
\end{theorem}

\begin{proof}
Fix a span $j$ and a token $t\in[\tau_{j-1},\tau_j)$. Since $U_j=Y_{<\tau_j}$ contains $A_t$ and its prefix state $S_t$, the score term is measurable with respect to $U_j$. By the tower property,
\begin{align}
&\mathbb{E}\left[R\nabla_\theta\log\pi_\theta(A_t\mid S_t)\right] \notag \\
&\quad =
\mathbb{E}\left[V(U_j)\nabla_\theta\log\pi_\theta(A_t\mid S_t)\right].
\label{eq:span-tower-proof}
\end{align}
For any random variable $B$ that is measurable with respect to $S_t$,
\begin{align}
&\mathbb{E}\left[B\nabla_\theta\log\pi_\theta(A_t\mid S_t)\right] \notag \\
&\quad =
\mathbb{E}\left[B\,\mathbb{E}\left[\nabla_\theta\log\pi_\theta(A_t\mid S_t)\mid S_t\right]\right]
=0,
\label{eq:baseline-zero-proof}
\end{align}
because $\mathbb{E}[\nabla_\theta\log\pi_\theta(A_t\mid S_t)\mid S_t]=0$. Taking $B=V(U_0)$ in \eqref{eq:baseline-zero-proof} and combining with \eqref{eq:span-tower-proof} gives the identity for $\mathrm{GA}(U_j)$. Taking $B=V(U_{j-1})$, which is measurable with respect to $S_t$ for every $t\in[\tau_{j-1},\tau_j)$, gives the identity for $\mathrm{LA}(U_j)$. Summing over $t=\tau_{j-1},\ldots,\tau_j-1$ proves $\mathbb{E}[\mathrm{GA}(U_j)Z_j]=\mathbb{E}[\mathrm{LA}(U_j)Z_j]=\mathbb{E}[RZ_j]$.

For the empirical claim, condition on $n(v)=m$ for a fixed prefix node $v=(p,u)$. By Assumption~\ref{ass:fixed-prefix-sampling}, the descendant rewards are conditionally independent samples from $R\mid Y_{<p}=u$. Their mean is $V(u)$ and their variance is $\operatorname{Var}(R\mid Y_{<p}=u)$. Since $\widehat V(v)=m^{-1}\sum_{i\in I(v)}r_i$, the conditional mean is $V(u)$ and the conditional variance is $\operatorname{Var}(R\mid Y_{<p}=u)/m$.
\end{proof}

\subsection{Proof of the fork-budget claim}

\begin{theorem}[Fork-budget support-resolution tradeoff]
\label{thm:fork-budget}
Fix rollout budget $K$, fork budget $B\ge1$, and retention threshold $m_0\ge2$, with $K\ge m_0$. Let $\mathcal{N}_{\mathrm{fork}}^{(m_0)}$ be the retained non-root prefix nodes obtained by applying at most $B$ selected boundaries to the $K$ sampled responses and keeping only prefix groups with at least $m_0$ descendants. Then $|\mathcal{N}_{\mathrm{fork}}^{(m_0)}|\le B\lfloor K/m_0\rfloor$. For a selected retained node $v=(p,u)$, define its selection-conditional value as $V_{\mathrm{sel}}(v)=\mathbb{E}[R_J\mid v\in\mathcal{N}_{\mathrm{fork}}^{(m_0)}]$, where $J$ is drawn uniformly from the descendant set $I(v)$. Under Assumption~\ref{ass:selected-node-concentration}, if $n_{\min}=\min_{v\in\mathcal{N}_{\mathrm{fork}}^{(m_0)}}n(v)$, then with probability at least $1-\delta$, $\max_{v\in\mathcal{N}_{\mathrm{fork}}^{(m_0)}}|\widehat V(v)-V_{\mathrm{sel}}(v)|\le R_{\max}\sqrt{\log(2B\lfloor K/m_0\rfloor/\delta)/(2n_{\min})}$. If selection does not distort the descendant reward law beyond conditioning on the retained prefix, then $V_{\mathrm{sel}}(v)=V(u)$.
\end{theorem}

\begin{proof}
Fix a selected boundary $p$. Exact prefix matching partitions the $K$ sampled responses into disjoint groups. Since each retained group has size at least $m_0$, there are at most $\lfloor K/m_0\rfloor$ retained groups at boundary $p$. Since at most $B$ boundaries are selected,
\[
|\mathcal{N}_{\mathrm{fork}}^{(m_0)}|
\le
B\left\lfloor\frac{K}{m_0}\right\rfloor .
\]

If $\mathcal{N}_{\mathrm{fork}}^{(m_0)}=\varnothing$, the high-probability statement is vacuous. Otherwise, let $n_{\min}=\min_v n(v)$ over retained nodes. By Assumption~\ref{ass:selected-node-concentration}, for any retained node $v$,
\[
\mathbb{P}\left(|\widehat V(v)-V_{\mathrm{sel}}(v)|\ge\epsilon\right)
\le
2\exp\left(-\frac{2n(v)\epsilon^2}{R_{\max}^2}\right).
\]
Since $n(v)\ge n_{\min}$, a union bound gives
\begin{align}
&\mathbb{P}\left(\max_v |\widehat V(v)-V_{\mathrm{sel}}(v)|\ge\epsilon\right) \notag \\
&\quad \le
2|\mathcal{N}_{\mathrm{fork}}^{(m_0)}|
\exp\left(-\frac{2n_{\min}\epsilon^2}{R_{\max}^2}\right) \notag \\
&\quad \le
2B\left\lfloor\frac{K}{m_0}\right\rfloor
\exp\left(-\frac{2n_{\min}\epsilon^2}{R_{\max}^2}\right).
\end{align}
Setting the last expression equal to $\delta$ and solving for $\epsilon$ proves the stated bound. Finally, if selection does not distort the descendant reward law beyond the retained prefix $u$, then $V_{\mathrm{sel}}(v)=\mathbb{E}[R\mid Y_{<p}=u]=V(u)$.
\end{proof}

\subsection{Proof of the prefix-matching claim}

\begin{theorem}[Collision-only prefix matching is shallow-biased]
\label{thm:prefix-matching}
Let $\mathcal{A}=\{p_1<\cdots<p_M\}$ be the admissible boundaries. For $p\in\mathcal{A}$, define $U_p=Y_{<p}$, $\kappa(p)=\sum_{u\in\mathcal{V}^p}\mathbb{P}(U_p=u)^2$, and $\eta(p)=\operatorname{Var}(\mathbb{E}[R\mid U_p])/\operatorname{Var}(R)$ when $\operatorname{Var}(R)>0$. Then $\kappa(p)$ is non-increasing in $p$. Hence any collision-only selector $p_{\mathrm{coll}}\in\arg\max_{p\in\mathcal{A}}\kappa(p)$ selects from the earliest plateau $\{p\in\mathcal{A}:\kappa(p)=\kappa(p_1)\}$; if $\kappa(p_1)>\kappa(p_2)$, it uniquely selects $p_1$. Moreover, there exist $p_0<p_\star$ such that $\kappa(p_0)>\kappa(p_\star)>0$, while $\eta(p_0)=0$ and $\eta(p_\star)=1$.
\end{theorem}

\begin{proof}
We first prove that $\kappa(p)$ is non-increasing. The prefix partition at depth $p+1$ refines the prefix partition at depth $p$. Fix a length-$p$ prefix $u$, and let $q_a=\mathbb{P}(U_{p+1}=(u,a))$. Then $\mathbb{P}(U_p=u)=\sum_a q_a$, and
\[
\sum_a q_a^2
\le
\left(\sum_a q_a\right)^2
=
\mathbb{P}(U_p=u)^2 .
\]
Summing over all $u\in\mathcal{V}^p$ gives $\kappa(p+1)\le\kappa(p)$. Thus, by induction, $p_a<p_b$ implies $\kappa(p_a)\ge\kappa(p_b)$.

It follows that any maximizer of $\kappa(p)$ over $\mathcal{A}=\{p_1<\cdots<p_M\}$ lies in the earliest plateau $\{p\in\mathcal{A}:\kappa(p)=\kappa(p_1)\}$. If $\kappa(p_1)>\kappa(p_2)$, then no deeper admissible boundary can match $\kappa(p_1)$, so $p_1$ is the unique maximizer.

It remains to show that high collision need not imply reward information. Consider two admissible boundaries $p_0<p_\star$. Suppose all responses share the same prefix at $p_0$. Then $U_{p_0}$ is deterministic, so $\kappa(p_0)=1$. Since $\mathbb{E}[R\mid U_{p_0}]$ is constant, $\eta(p_0)=0$.

At boundary $p_\star$, suppose there are two prefixes $u_0,u_1$, each with probability $1/2$, and let $R=\mathbf{1}\{U_{p_\star}=u_1\}$. Then $\kappa(p_\star)=(1/2)^2+(1/2)^2=1/2$, so $\kappa(p_0)>\kappa(p_\star)>0$. Also, $\mathbb{E}[R\mid U_{p_\star}=u_0]=0$ and $\mathbb{E}[R\mid U_{p_\star}=u_1]=1$, so $\mathbb{E}[R\mid U_{p_\star}]=R$. Hence $\operatorname{Var}(\mathbb{E}[R\mid U_{p_\star}])=\operatorname{Var}(R)$, and $\eta(p_\star)=1$. Therefore collision-only selection can prefer a shallower boundary with larger collision but zero reward informativeness.
\end{proof}

\end{document}